\PassOptionsToPackage{table,dvipsnames}{xcolor}
\documentclass[10pt,twocolumn,letterpaper]{article}
\usepackage{times}
\usepackage{epsfig}
\usepackage{graphicx}
\usepackage{amsmath}
\usepackage{amssymb}
\usepackage{multirow}
\usepackage{amssymb}
\usepackage{pifont}

\usepackage{adjustbox}
\usepackage{amsthm} 
\newtheorem{theorem}{Theorem} 
\def\tn{\Tilde{n}}
\def\LL{\mathcal{L}}
\def\tS{\tilde{S}}
\def\E{\mathbb{E}}

\usepackage{colortbl}
\usepackage{graphicx}
\definecolor{globalcolor}{RGB}{255, 102, 0}
\definecolor{localcolor}{RGB}{0, 153, 76} 
\definecolor{lightgray}{gray}{0.9}

\usepackage{tcolorbox}
\usepackage{listings}
\usepackage{pict2e} 

\tcbuselibrary{breakable}

\usepackage{subcaption} 

\usepackage{iccv}              

\definecolor{iccvblue}{rgb}{0.21,0.49,0.74}
\usepackage[pagebackref,breaklinks,colorlinks,allcolors=iccvblue]{hyperref}

\def\paperID{*****} 
\def\confName{ICCV}
\def\confYear{2025}

\title{Frequency-Semantic Enhanced Variational Autoencoder for Zero-Shot Skeleton-based Action Recognition}

\author{
Wenhan Wu$^{1}$, Zhishuai Guo$^{2}$, Chen Chen$^{3}$, Hongfei Xue$^{1}$, Aidong Lu$^{1}$\\
$^{1}$University of North Carolina at Charlotte, Department of Computer Science\\
$^{2}$Northern Illinois University, Department of Computer Science\\
$^{3}$Center for Research in Computer Vision, University of Central Florida\\
{\tt\small \{wwu25, hongfei.xue, aidong.lu\}@charlotte.edu}, {\tt\small zguo@niu.edu}, {\tt\small chen.chen@crcv.ucf.edu}
}

\begin{document}

\maketitle
\begin{abstract}
Zero-shot skeleton-based action recognition aims to develop models capable of identifying actions beyond the categories encountered during training. Previous approaches have primarily focused on aligning visual and semantic representations but often overlooked the importance of fine-grained action patterns in the semantic space (e.g., the hand movements in drinking water and brushing teeth). To address these limitations, we propose a \textbf{F}requency-\textbf{S}emantic Enhanced \textbf{V}ariational \textbf{A}uto\textbf{e}ncoder (\textbf{FS-VAE}) to explore the skeleton semantic representation learning with frequency decomposition. FS-VAE consists of three key components: 1) a frequency-based enhancement module with high- and low-frequency adjustments to enrich the skeletal semantics learning and improve the robustness of zero-shot action recognition; 2) a semantic-based action description with multilevel alignment to capture both local details and global correspondence, effectively bridging the semantic gap and compensating for the inherent loss of information in skeleton sequences; 3) a calibrated cross-alignment loss that enables valid skeleton-text pairs to counterbalance ambiguous ones, mitigating discrepancies and ambiguities in skeleton and text features, thereby ensuring robust alignment. Evaluations on the benchmarks demonstrate the effectiveness of our approach, validating that frequency-enhanced semantic features enable robust differentiation of visually and semantically similar action clusters, thereby improving zero-shot action recognition. Our project is publicly available at: \textcolor{magenta}{\url{https://github.com/wenhanwu95/FS-VAE}}.
\end{abstract}

\section{Introduction}

\begin{figure}[htp]
\vspace{-5pt}
  \centering
  \includegraphics[width=1.0\linewidth]{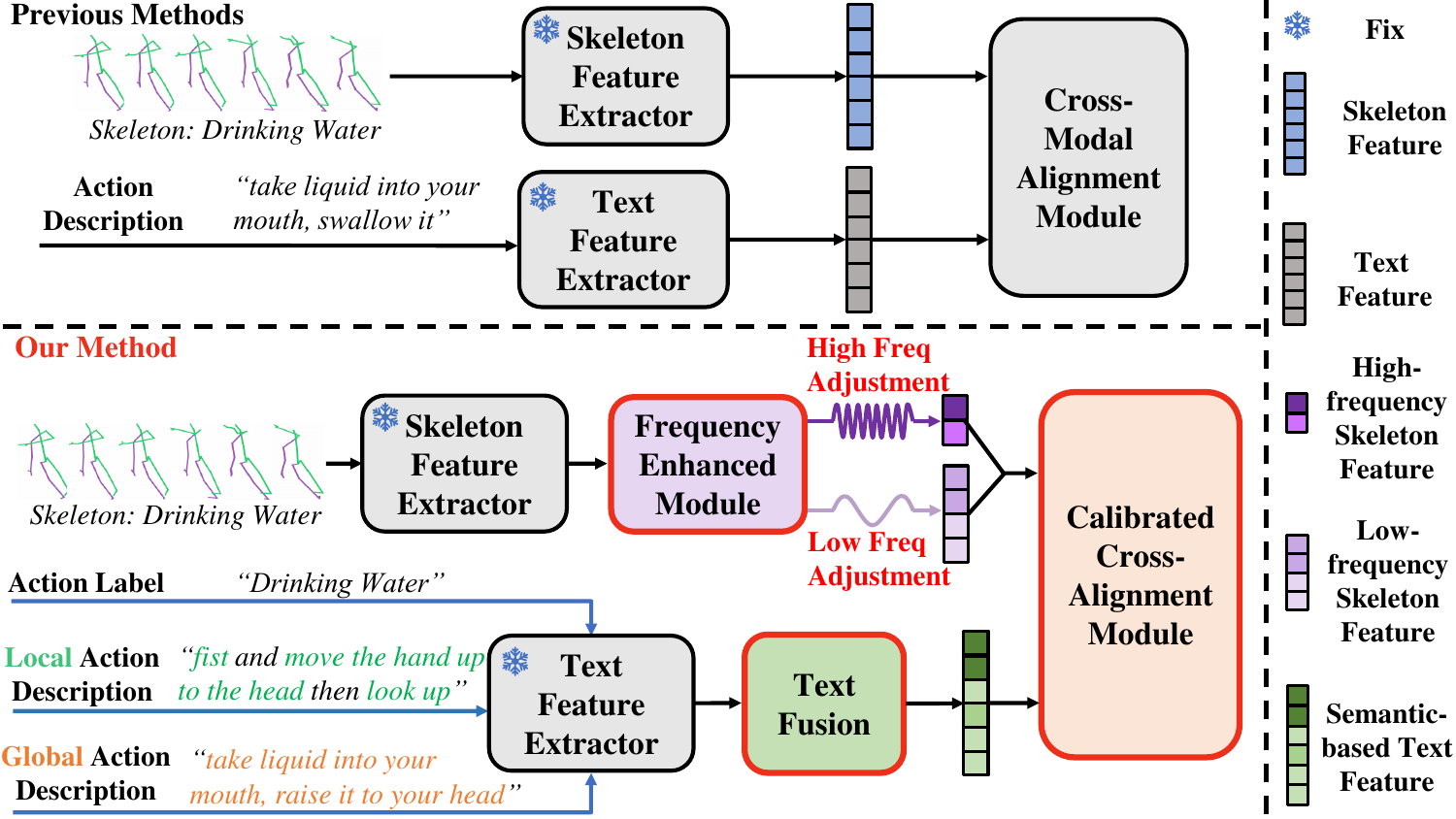}
  \vspace{-15pt}
  \caption{The overall design of \textcolor{red}{our} frequency-semantic enhanced variational autoencoder for zero-shot skeleton action recognition.}
  \label{fig:fig1}
  \vspace{-18pt}
\end{figure}

Human action recognition has gained significant attention in computer vision due to its wide-ranging applications, including surveillance \cite{singh2019multi, shorfuzzaman2021towards}, human-computer interaction \cite{nayak2021human, kashef2021smart}, and automated driving \cite{nayak2021human, kashef2021smart}. Among the various modalities, skeleton-based action recognition stands out for its robustness to environmental variations, as it focuses on 3D structural poses derived from human joints. Although conventional supervised approaches \cite {hou2016skeleton, ke2017new, zhang2017geometric, lee2017ensemble, chi2022infogcn, chen2021channel, liu2023transkeleton, xin2023skeleton, wu2024frequency} have achieved remarkable success in this domain, their reliance on extensive labeled data limits their performance to diverse and unseen action categories. Zero-shot skeleton-based action recognition (ZSSAR) addresses this challenge by enabling models to recognize unseen actions using knowledge from seen action categories and semantic descriptions.

Existing ZSSAR methods \cite{hubert2017learning, gupta2021syntactically, zhou2023zero,li2023multi,chen2024fine, zhu2024part,li2025sa} primarily align skeleton features with text embeddings within a shared latent space to enhance generalization. Specially, \cite{gupta2021syntactically, li2023multi, li2025sa} employ Variational Auto Encoder (VAE) \cite{kingma2013auto} as the training framework to learn structured and generalizable latent representations.
However, these VAE-based approaches often yield less semantic information due to their coarse feature representations. Skeleton sequences, unlike raw videos, lack detailed appearance cues, making it difficult to encode the fine-grained semantics. As a result, critical motions (e.g., nuanced limb or hand movements) are often underrepresented, limiting the models' ability to capture semantic distinctions between similar actions such as ``drinking" and ``eating". Additionally, skeleton data are inherently ambiguous due to occlusions and variations in camera viewpoints, further complicating motion interpretation. Traditional cross-modal alignments often treat all skeleton-text pairs as equally reliable, ignoring the uncertainty in skeleton representations. Since skeleton features can be noisy and ambiguous, rigid alignment \cite{gupta2021syntactically, li2023multi, li2025sa} with text embeddings may lead to misalignment and degraded generalization in zero-shot scenarios.  

Driven by these concerns, we raise two fundamental questions: \textit{\textbf{Q1}: How can skeletal semantics be enriched to enhance the generalization of learned features? \textbf{Q2}: How can cross-modal alignment be improved by effectively leveraging enriched semantics?}
To address these challenges, \textit{we design a framework that enhances action semantics with frequency-based modeling and semantic action descriptions. Additionally, a calibrated cross-modal alignment module is proposed to bridge modality gaps, enabling robust zero-shot recognition of both global and fine-grained patterns.}

\textbf{Firstly}, our approach introduces a Frequency Enhanced Module, which employs the Discrete Cosine Transform (DCT) \cite{ahmed1974discrete} to transform and enhance skeleton motions in the frequency domain. This decomposition enables a structured enhancement strategy, where low-frequency components (overall semantic structure of actions) and high-frequency components (fine-grained motion details) are selectively refined. Specifically, low-frequency coefficients undergo a progressively diminishing amplification effect to strengthen the global motion representation without distorting structural integrity. Meanwhile, high-frequency coefficients are subjected to an adaptive attenuation mechanism that gradually reduces their magnitude without excessive suppression. This adjustment allows us to preserve subtle actions, such as limb movements and micro-gestures, and simultaneously mitigates the influence of high-frequency noise caused by skeletal jitter. The enhancement not only enriches the semantic representations but also improves the model's robustness against noise and irrelevant variations.


\textbf{Secondly}, our framework incorporates a Semantic-based action Description (SD) mechanism to generate embeddings that capture both localized and global action semantics, enhancing cross-modal alignment in ZSSAR. This allows the model to leverage semantic consistency for recognizing unseen actions without direct supervision. The SD consists of Local action Description (LD), which encodes fine-grained motion details, and Global action Description (GD), which represents overall body posture and movement patterns. For example, LD highlights hand movement in ``drinking water'', while GD captures body coordination and the sequential flow of motion. This structured description ensures that both detailed actions and contextual motion are well-represented, enabling a precise alignment with frequency-enhanced skeleton features. 

\textbf{Thirdly}, a calibrated cross-alignment loss is proposed for semantic embeddings with frequency-enhanced skeleton features, addressing modality gaps and skeleton ambiguities in ZSSAR tasks. It minimizes the disparity between true skeleton-semantic pairs while mitigating mismatched pairs. Unlike conventional uniform alignment losses, the calibrated loss employs a sigmoid-based distance measure to dynamically balance contributions from positive and negative pairs, ensuring robust learning even in the presence of noisy or ambiguous skeleton data. By integrating the frequency-enhanced structure of skeleton and text features, the loss further reinforces cross-modal correspondence. Actions with overlapping semantics, such as ``drinking" and ``eating", can be effectively distinguished by leveraging fine-grained details (e.g., hand trajectories) and global patterns (e.g., body movements). This approach mitigates overfitting to noisy alignments and reduces modality-specific biases, improving generalization to unseen actions in zero-shot scenarios. The overall method of our FS-VAE is illustrated in Fig.~\ref{fig:fig1}, and the contributions are as follows:

\begin{itemize} 
\item We propose a \textbf{Frequency Enhanced Module} that employs Discrete Cosine Transform (DCT) to decompose skeleton motions into high- and low-frequency components, allowing adaptive feature enhancement to improve semantic representation learning in ZSSAR.

\item We introduce a novel \textbf{Semantic-based action Description (SD)}, comprising Local action Description (LD) and Global action Description (GD), to enrich the semantic information for improving the model performance.

\item A \textbf{Calibrated Cross-Alignment Loss} is proposed to address modality gaps and skeleton ambiguities by dynamically balancing positive and negative pair contributions. This loss ensures robust alignment between semantic embeddings and skeleton features, improving the model's generalization to unseen actions in ZSSAR.

\item Extensive experiments on benchmark datasets demonstrate that our framework significantly outperforms state-of-the-art methods, validating its effectiveness and robustness under various seen-unseen split settings.
\end{itemize}

\section{Related Works}
\subsection{Zero-Shot Skeleton Action Recognition}
Traditional ZSSAR methods focus mainly on mapping skeleton features and semantic embeddings into a shared latent space for alignment \cite{hubert2017learning, wray2019fine, gupta2021syntactically, zhou2023zero}. These approaches utilize techniques such as visual-textual correlation learning and adversarial training to reduce the modality gap. Recent works have explored enhanced feature representations and multi-modal alignment strategies to improve performance. For example, \cite{zhu2024part} leverages part-based feature modeling to address prompting and partitioning issues in alignment, while \cite{li2025sa} introduces semantic attention mechanisms to highlight irrelevant and related semantic features. Despite these advances, existing methods often overlook the semantics of frequency-domain features in capturing both fine-grained motions and global action patterns. Moreover, ambiguities in skeleton representations and noisy or mismatched skeleton-text pairs remain significant challenges.

Unlike prior works that focus solely on the spatial and temporal domain, our work differs fundamentally by leveraging frequency decomposition to model and enhance skeleton motions in the frequency domain, capturing both fine-grained motion variations and overarching action structure to provide richer skeletal semantics. Furthermore, our calibrated cross-alignment loss explicitly addresses skeleton-text ambiguities and modality gaps, ensuring robust alignment in zero-shot scenarios. 

\subsection{Skeleton-based Frequency Representation Learning}
Pose-based approaches aim to directly extract motion patterns from human poses for applications such as motion prediction \cite{li2020dynamic}, pose estimation \cite{ zheng20213d}, and action recognition \cite{chi2022infogcn}. These methods rely on representations from the pose space, which naturally encode spatial structural relationships and temporal motion dependencies. However, effectively integrating these spatio-temporal aspects into a unified framework remains a significant challenge.

Recent research has taken advantage of frequency domain transformations to encode temporal information \cite{akhter2008nonrigid} compactly and smoothly. Studies such as \cite{mao2019learning, wang2021multi, li2022skeleton, gao2023decompose} utilize DCT to convert temporal motion signals into the frequency domain, facilitating frequency-specific representation learning. The decomposition of motion signals into high- and low-frequency components enables a fine-grained action analysis, preserving both subtle motion details and global movement patterns. Despite the success of frequency-based modeling, its application in skeleton-based action recognition remains limited and has not yet been explored in the context of zero-shot learning (ZSL). For example, \cite{chang2024wavelet} employed a wavelet transform-based approach to disentangle salient and subtle motion features, targeting fine-grained action recognition. Similarly, \cite{wu2024frequency} proposed a frequency-aware transformer that enhances discriminative feature learning for fully supervised action recognition.

In contrast, our approach pioneers the use of DCT in ZSSAR, leveraging its ability to effectively enhance and redistribute motion signals across frequency coefficients. This ensures a robust representation of global motion patterns while preserving fine-grained movement details without amplifying noise. Additionally, a semantic-based action description further enriches action semantics by capturing both localized and holistic action patterns, bridging the semantic gap between textual and skeletal features.

\section{FS-VAE: Frequency-Semantic Enhanced Variational Autoencoder}
Our goal is to recognize actions from unseen categories using only seen class knowledge and their semantic representations. We adopt a generative VAE framework \cite{gupta2021syntactically, li2023multi, li2025sa} to learn the skeleton and semantic cross-model features, which are used to generate unseen class representations in latent space. To further enhance generalization in ZSL, we propose a frequency-semantic enhanced framework that refines skeletal inputs through frequency decomposition, preserving essential motion patterns while improving alignment with semantic features. Below, we introduce the model details in FS-VAE.

\begin{figure*}[htp]
\vspace{-5pt}
  \centering
  \includegraphics[width=1\linewidth]{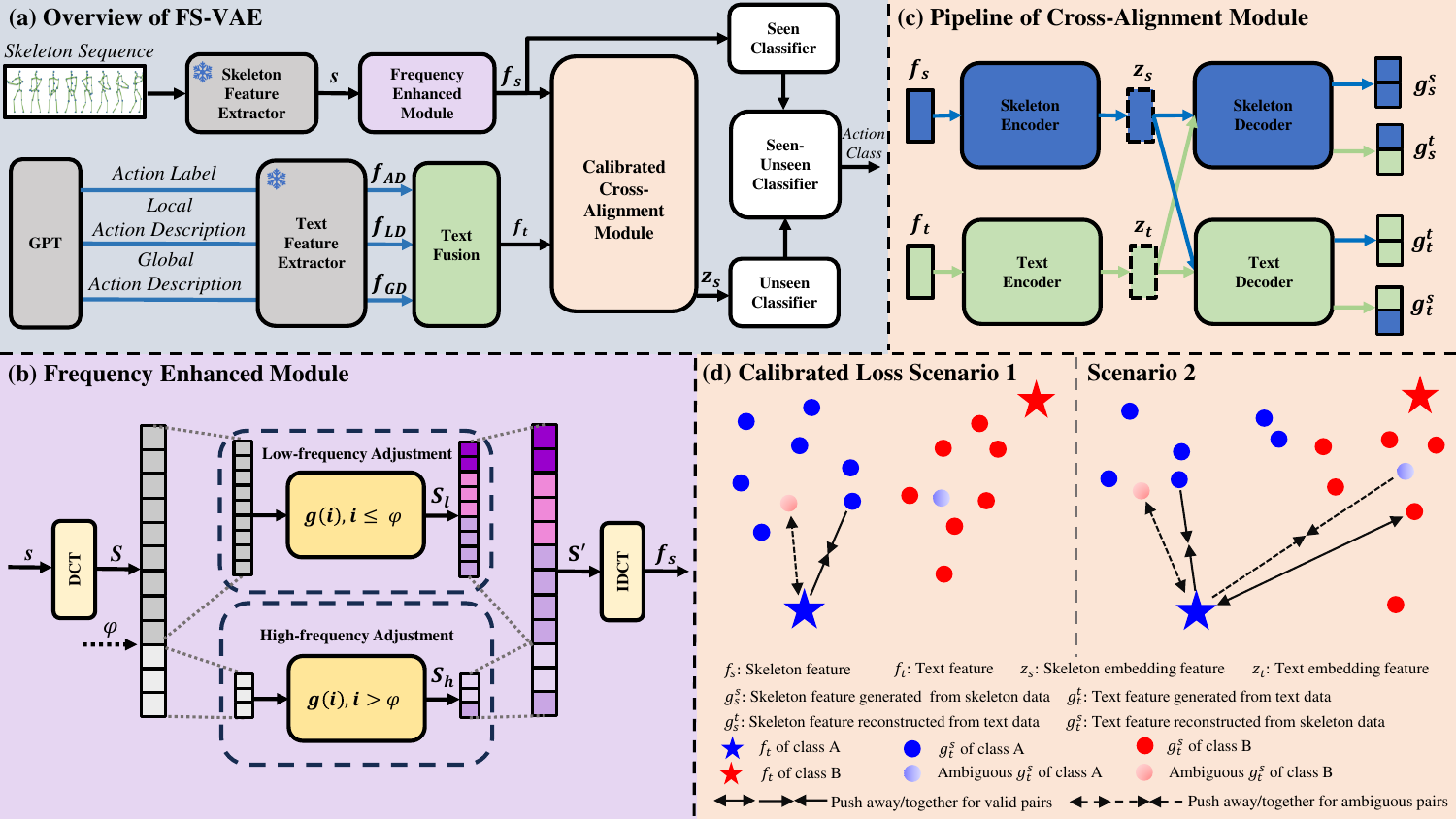}
  \vspace{-15pt}
  \caption{\small{Overview of the proposed FS-VAE. The frequency-enhanced module integrates the global and fine-grained skeleton utilizing the low-frequency and high-frequency adjustments introduced in Section \ref{sec:frequency}. The semantic-based action descriptions, including action labels, local action descriptions, and global action descriptions, are introduced in Section \ref{sec:semantic} to generate comprehensive semantic embeddings for cross-alignment. Moreover, the novel calibrated loss in the cross-alignment module is proposed in Section \ref{sec:cali} for minimizing the disparity between semantic and skeletal features.}}
  \label{fig:fig2}
  \vspace{-16pt}
\end{figure*}
\subsection{Problem Formulation}

Zero-shot skeleton-based action recognition (ZSSAR) aims to classify actions from unseen categories using knowledge from seen categories. A skeleton dataset is represented as \( \mathcal{D} = \{(\mathbf{X}_i, \mathbf{C}_i, \mathbf{A}_i)\}_{i=1}^{N} \), where \( \mathbf{X}_i \in \mathbb{R}^{J \times 3 \times F \times M} \) denotes the skeleton sequence of the \(i\)-th sample, composed of 3D joint coordinates over \(F\) frames for \(J\) joints and \(M\) subjects. The corresponding action category is \(\mathbf{C}_i\), and \( \mathbf{A}_i \) represents the GPT-generated semantic description. The dataset is partitioned into a training set \( \mathcal{D}_{\text{tr}}^s \) with samples of seen categories \( \mathcal{C}_s \), and two disjoint test sets: \( \mathcal{D}_{\text{te}}^u \) for unseen categories \( \mathcal{C}_u \) and \( \mathcal{D}_{\text{te}}^s \) for seen categories. The category sets satisfy \( \mathcal{C} = \mathcal{C}_s \cup \mathcal{C}_u \) and \( \mathcal{C}_s \cap \mathcal{C}_u = \emptyset \), each action category is associated with \( \mathbf{A}_i \).

The objective of ZSSAR is to learn a mapping function \( f: \mathbb{R}^{J \times 3 \times F \times M} \to \mathcal{C}_u \). To achieve this, we align the skeleton feature \( f_s \) extracted from \( \mathbf{X}_i \) with the text feature \( f_t \) obtained from \( \mathbf{A}_i \) in a shared latent space. During training, a feature alignment mechanism enforces the relationship between \( f_s \) and \( f_t \) for seen categories \( \mathcal{C}_s \). Specifically, the skeleton feature \( f_s \) is encoded into a latent distribution \( z_s \), while the text feature \( f_t \) is mapped to another latent distribution \( z_t \). These latent variables \( z_s \) and \( z_t \) serve as a bridge for cross-modal knowledge transfer to ensure the learned representations capture both motion and semantic patterns within the cross-alignment. In Generalized Zero-Shot Skeleton-Based Action Recognition (GZSSAR), the model classifies actions from both \( \mathcal{C}_s \) and \( \mathcal{C}_u \) during testing.

\subsection{Frequency Enhanced Module}
\label{sec:frequency}
\textbf{Motivation for Frequency Enhancement in ZSL.}
In fully supervised learning, frequency-aware models \cite{chang2024wavelet,wu2024frequency}  capture both high- and low-frequency components from labeled data, where high-frequency details are particularly useful for recognizing subtle movements, and low-frequency motions are utilized to capture global movement patterns.  However, the lack of unseen class data prevents the direct learning of class-specific high-frequency distributions in ZSL, making these details more sparse and noisy. To address this, our approach enhances low-frequency components to extract richer semantic information, improving generalization to unseen categories. Meanwhile, we adaptively suppress high-frequency variations to preserve essential fine-grained details while mitigating noise (e.g., skeletal jitter or limb fluctuations). This adjustment reinforces the ZSL training to learn richer skeletal information and more structured semantics compared to conventional purely spatial-temporal modeling \cite{gupta2021syntactically, li2023multi, li2025sa}, which leads to more effective ZSSAR performance. The pipeline of the enhanced module is illustrated in Fig. \ref{fig:fig2} (b).

\textbf{Frequency Division Formula.}
Let \( \mathbf{s} \in \mathbb{R}^{J \times C \times F} \) denote the input joint sequence, where \( J \) represents the number of joints, \( C \) the coordinate dimension (e.g., $x$, $y$, $z$), and \( F \) the number of frames. The trajectory of the $j$-th joint across $T$ frames is denoted as \( T_j = (t_{j,1}, t_{j,2}, \ldots, t_{j,F}) \).
We apply the DCT \cite{mao2019learning, zhao2023poseformerv2, wu2024frequency} to obtain the frequency-domain representation for skeleton sequence: $\mathbf{S} = \text{DCT}(\mathbf{s})$, the DCT decomposes the input skeleton sequence \textbf{s} into frequency components, producing the transformed representation 
\textbf{S} of the same length as the input. Each component in \textbf{S} corresponds to a specific frequency coefficient, where lower-indexed coefficients represent low-frequency (global) motion patterns, and higher-indexed coefficients capture high-frequency (fine-grained) details. For the trajectory \( T_j \), the $i$-th DCT coefficient to each individual trajectory is calculated as:
\begin{equation}
\resizebox{0.85\hsize}{!}{$
C_{j,i} = \sqrt{\frac{2}{F}} \sum_{f=1}^{F} t_{j,f} \frac{1}{\sqrt{1 + \delta_{i1}}} 
\cos\left[\frac{\pi(2f - 1)(i - 1)}{2F}\right]
$}
\label{eq:dct_coeff}
\end{equation}
where the \textit{Kronecker delta} $ \delta_{ij} = 1$ if $i = j$, and $ \delta_{ij} = 0$ otherwise. In particular, \( i \in \{1, 2, \ldots, F \} \), and the larger $i$ corresponds to higher frequency coefficients. These coefficients enable us to represent skeleton motion effectively within the frequency domain by capturing both subtle dynamic details and global motion patterns \cite{wu2024frequency}. 

\textbf{Low-Frequency Adjustment.}  
For the low-frequency range, the adjustment is applied using the piecewise scaling function \( g(i) \):
\begin{equation}
\mathbf{S}_{l} \leftarrow \mathbf{S} \cdot g(i), i \leq \varphi
\end{equation}
where \( g(i) = 1 + w_i\left(1 - \frac{i}{b}\right) \) for low-frequency components,  \( \varphi \) is the low frequency threshold. \( \mathbf{S}_{l} \) represents the low-frequency components, which capture global motion patterns such as large-scale movements of limbs and torso. The term \( b \) is an adjusting parameter intended to make \(g(i)\) decrease gradually within the low-frequency range. The fraction \( \frac{i}{b} \) ensures a progressive reduction in enhancement strength as frequency increases, thereby maintaining the integrity of large-scale global motion. The learned weight \( w_i \) adaptively controls enhancement for different frequencies, amplifying the most distinguishing ones.

\textbf{High-Frequency Adjustment.}  
For the high-frequency range, the scaling function \( g(i) \) is given by:
\begin{equation}
\mathbf{S}_{h} \leftarrow \mathbf{S} \cdot g(i), i > \varphi
\end{equation}
where \( g(i) = 1 - w_i\left(1 - \frac{i - b}{b}\right) \). \( \mathbf{S}_{h} \) represents the high-frequency components of the skeleton features, which captures fine-grained details such as finger, wrist, and rapid limb movements. In high-frequency adjustment, \( b \) serves as a normalization factor that controls the suppression of high-frequency variables. Specifically, it ensures that attenuation decreases (i.e., $g(i)$ increases) smoothly as frequency increases, preventing excessive suppression of fine-grained motion details. By scaling the suppression term proportionally to \( i - b \), this formulation mitigates skeletal noise while preserving essential micro-movements. Meanwhile, \(w_i\) adaptively modulates the suppression strength for each high-frequency component, up-weighting the most distinguishing frequencies.

\textbf{Inverse Transform.}
The adjusted frequency-domain signal is reconstructed into the time domain using the Inverse Discrete Cosine Transform (IDCT), represented as: $f_{s} = \text{IDCT}(\mathbf{S'})$,
where $\mathbf{S'}$ is the frequency-enhanced skeleton component in the frequency domain. The specific restoration for each joint trajectory is given by:
\begin{equation}
\resizebox{0.85\hsize}{!}{$
t_{j,f} = \sqrt{\frac{2}{F}} \sum_{i=1}^{F} C_{j,i} \frac{1}{\sqrt{1 + \delta_{i1}}} 
\cos\left[\frac{\pi(2f - 1)(i - 1)}{2F}\right]
$}
\label{eq:idct}
\end{equation}
where $j \in \{1, 2, \ldots, J \}$ and $f \in \{1, 2, \ldots, F\}$. Here, $t_{j,f}$ represents the restored joint trajectory in the time domain, reconstructed from its frequency-domain coefficients \( C_{j,i} \).

This process integrates enhanced global patterns and preserved fine-grained details, creating a comprehensive representation of the action. Using the enhanced time domain skeleton feature \( f_s \), the model aligns these features with semantic embeddings, enabling robust recognition in zero-shot scenarios. More frequency analysis and method illustration can be found in \textcolor{magenta}{Appendix F}. 

\subsection{Semantic-based Action Description}
\label{sec:semantic}
Unlike the previous ZSL methods that focus on motion descriptions in a temporal way \cite{li2023multi} and focus on the action prompting with GPTs \cite{zhu2024part} that ignores the semantic characteristics, we propose a novel strategy that leverages semantic decomposition and feature alignment to fully capture both the localized details and global semantic structures inherent in human actions. This method stems from the observation that actions can be naturally divided into components reflecting dynamic movements and overarching patterns, enabling a comprehensive and robust representation of action semantics.

The semantic text description consists of the  Action Label (AL) and a Semantic-Based Description (SD). SD is further divided into two complementary components: Local action Description (LD) and Global action Description (GD). We adopt the pre-trained text-encoder of CLIP \cite{radford2021learning} to extract the corresponding semantic features. \(f_{\text{AL}}\) indicates the feature of action label. 
The local component \( f_{\text{LD}} \) captures fine-grained motion details, which are crucial for understanding localized dynamics and specific body-part interactions. For example, in the action of ``drinking water,'' \( f_{\text{LD}} \) describes detailed movements such as ``fist and move the hand up to the head, then look up,'' to emphasize precise body-part motions. In contrast, \( f_{\text{GD}} \) represents the overall movements that provide a high-level overview, such as ``take liquid into your mouth, raise it to your head.'' By integrating \(f_{\text{AL}}\), \( f_{\text{LD}} \) and \( f_{\text{GD}} \), our representation enriches the semantic space, capturing both localized motion details and holistic action patterns. Extra examples of action descriptions and promptings are listed in \textcolor{magenta}{Appendix D}.

To unify these components into a cohesive semantic embedding, we concatenate the descriptions and normalize them as follows:
\begin{equation}
f_{\text{t}} = \frac{\text{Concat}(f_{\text{AL}}, f_{\text{LD}}, f_{\text{GD}})}{\|\text{Concat}(f_{\text{AL}}, f_{\text{LD}}, f_{\text{GD}})\|}
\end{equation}

\subsection{Calibrated Cross-Alignment Loss}
\label{sec:cali}

\textbf{Motivation for Calibrated Loss.}
Text data, especially the semantically rich descriptions we introduce in Section \ref{sec:semantic}, are inherently clean and precise, offering a strong foundation for capturing action nuances.
In zero-shot learning, the text features serve as the bridge between seen and unseen classes, enabling the model to generalize effectively. 
However, when the text encoder is affected by noisy skeleton data, it may struggle to retain these semantic details, leading to suboptimal performance. 

Skeleton-based features (e.g., \(g^s_t\)), on the other hand, are noisy and unreliable for the following reasons. First, skeleton features omit crucial contextual information from the raw video data, including environmental context and fine-grained motion details.
Second, variations in camera angles and viewpoints further exacerbate ambiguities in skeletons. 
Finally, skeletons for actions with similar motion patterns, such as ``drinking water" and ``eating a meal," are inherently difficult to distinguish. 

As a result, rigidly aligning \(g^s_t\) with the text features \(f_t\) may result in poor updates to the text encoder, causing the model to overlook important semantic details in the text data. This misalignment is particularly problematic in zero-shot learning, where the model must generalize to unseen classes based on robust feature representations. 
To this end, we propose the following calibrated alignment loss that enhances the resilience of the text encoder to noise, preserving the quality of learned representations. The illustration is presented in Fig. \ref{fig:fig2} (c) and (d).


\textbf{Definition of Calibrated Loss.}
The calibrated loss adjusts the alignment by encouraging positive text-skeleton pairs to align while penalizing negative pairs, unlike \cite{li2023multi,li2025sa}, which only encourages the alignment of positive pairs. The calibrated loss is defined as:
{\scriptsize
\begin{equation}
\begin{split}
\mathcal{L}_{\text{Align}} = &\frac{\lambda}{B} \sum\limits_{i\in B} \frac{1}{1 + \exp((\|f_t(i) - g^s_t(i^-)\|^2-\|f_t(i) - g^s_t(i)\|^2)/\lambda)} \\
+ &\frac{\lambda}{B} \sum\limits_{i\in B} \frac{1}{1 + \exp((\|f_s(i) - g^t_s(i^-)\|^2-\|f_s(i) - g^t_s(i)\|^2)/\lambda)},
\end{split}
\end{equation}}
where \(i^-\) denotes a negative sample to $i$ in the batch, and $\lambda$ is a temperature parameter that controls the sensitivity of the alignment loss. 


\textbf{Key Scenarios Addressed.}
The calibrated loss is robust to the following scenarios:

1. Mismatched Positive Pair with a Reliable Negative Pair: A mismatched positive pair arises when a skeleton feature (e.g., \(g^s_t(i)\)) is inherently ambiguous and resembles skeletons from other classes. For instance, the skeletal sequence for ``reading" may be similar to ``writing." Aligning such an ambiguous skeleton with the text (\(f_t(i)\)) of its own class (``reading") can degrade the text encoder. This problem can be balanced by introducing a reliable negative pair,  i.e., a reliable negative pair \((f_t(i)\) and \(g^s_t(i^-)\) ((e.g., the text feature for ``reading" and a skeleton from ``writing" that is clearly distinguished from with ``reading")). 

2. Mismatched Positive Pair with Mismatched Negative Pair: The calibrated loss remains robust even when both positive and negative pairs are mismatched. For instance, \(f_t(i)\) and \(g^s_t(i)\) form a mismatched positive pair of text feature for ``reading" but the skeleton feature though labeled ``reading", resembles ``writing". Similarly, \(f_t(i)\) and \(g^s_t(i^-)\) form a mismatched negative pair where the text feature represents ``reading" and a skeleton is labeled ``writing" but inherently similar to ``reading". The calibrated loss leverages correctly aligned pairs in the dataset to counteract these inconsistencies.
This robustness stems from the symmetric property of the sigmoid function, i.e., \(\ell(a) + \ell(-a) = 1\), which has been utilized to handle noisy labels \cite{charoenphakdee2019symmetric,guo2023fedxl}.
In our case, let $a = (\|f_t(i) - g^s_t(i^-)\|^2-\|f_t(i) - g^s_t(i)\|^2)/\lambda$ represent a term where both pairs are mismatched. Correctly aligned pairs in the data set may contribute a corresponding term $-a$, leading to a natural counterbalance due to the symmetry of $\ell(\cdot)$. 
For non-symmetric losses, i.e., where $\ell(a) + \ell(-a)$ is not a constant, even a correctly aligned term cannot fully offset an incorrect one. 
A detailed analysis of the calibrated loss is provided in 
\textcolor{magenta}{Appendix E}. Notably, while our calibrated loss shares similarity with triplet losses \cite{schroff2015facenet,hermans2017defense,dong2018triplet,hoffer2015deep,do2019theoretically,kumar2016learning,ge2018deep}, most of these do not satisfy the symmetric property and are, therefore, less robust to noisy scenarios.  In \textcolor{magenta}{Appendix E}, we construct various alignment losses based on triplet loss and present experiments that demonstrate the advantages of our calibrated loss.
 
\textbf{Overall Loss.}
Following the literature on variational autoencoder (VAE)-based architecture for skeleton recognition \cite{li2023multi,gupta2021syntactically}, we use the evidence lower bound loss (ELBO) for reconstruction:
\begin{equation}
\begin{aligned}
\mathcal{L}_{\text{VAE}}^s = \ & \mathbb{E}_{q_\phi(z_s|f_s)}[\log p_\theta(f_s|z_s)] \\
& - \beta D_{KL}(q_\phi(z_s|f_s) \| p_\theta(z_s|f_s)),
\end{aligned}
\end{equation}
where $p_\theta(\cdot)$ and $q_\phi(\cdot)$ represent the likelihood and the prior, respectively. $\beta$ is a hyperparameter that controls the balance between reconstruction and regularization. $q_\phi(z_s|f_s)$ follows the multivariate Gaussian distribution $\mathcal{N}_s(\mu_s, \Sigma_s)$. $L_{VAE}^t$ is symmetric to $L_{VAE}^s$. And thus, the VAE loss is $L_{VAE} = L_{VAE}^s + L_{VAE}^t$. 
The overall loss is
\begin{equation}
    \mathcal{L}_{VAE}^{cali} =  \mathcal{L}_{VAE} + \alpha \mathcal{L}_{Align},
\end{equation}
where $\alpha$ adjusts the trade-off between the VAE loss and the alignment loss. 

\begin{table}[b]
\scriptsize
\renewcommand\arraystretch{1.1}
\centering
\vspace{-10pt}
\caption{Zero-Shot Learning Results. The highest values are highlighted in red, while the second-highest values (from other works) are marked in blue. \textcolor{red}{↑} indicates the improvement over the second-highest value. * indicates the reproduced results of the released codes. $^\dag$ denotes the use of only \(w_i\) for frequency coefficients.}
\vspace{-5pt}
\setlength\tabcolsep{3.0pt} 
{
\begin{tabular}{c|c|cc|cc}
\hline
\multirow{2}{*}{Methods} & \multirow{2}{*}{Venue} & \multicolumn{2}{c|}{NTU-60 (ACC,\%)} & \multicolumn{2}{c}{NTU-120 (ACC,\%)} \\ \cline{3-6} 
 &  & \multicolumn{1}{c|}{55/5 split} & 48/12 split & \multicolumn{1}{c|}{110/10 split} & 96/24 split \\ \hline
ReViSE\cite{hubert2017learning} & ICCV2017 & \multicolumn{1}{c|}{53.9} & 17.5 & \multicolumn{1}{c|}{55.0} & 32.4 \\
JPoSE\cite{wray2019fine} & ICCV2019 & \multicolumn{1}{c|}{64.8} & 28.8 & \multicolumn{1}{c|}{51.9} & 32.4 \\
CADA-VAE\cite{schonfeld2019generalized} & CVPR2019 & \multicolumn{1}{c|}{76.8} & 29.0 & \multicolumn{1}{c|}{59.5} & 35.8 \\
SynSE\cite{gupta2021syntactically} & ICIP2021 & \multicolumn{1}{c|}{75.8} & 33.3 & \multicolumn{1}{c|}{62.7} & 38.7 \\
SMIE\cite{zhou2023zero} & ACMM2023 & \multicolumn{1}{c|}{78.0} & 40.2 & \multicolumn{1}{c|}{61.3} & 42.3 \\
STAR\cite{chen2024fine} & ACMM2024 & \multicolumn{1}{c|}{81.4} & 45.1 & \multicolumn{1}{c|}{63.3} & 44.3 \\
GZSSAR*\cite{li2023multi} & ICIG2023 & \multicolumn{1}{c|}{\textcolor{blue}{83.3}} & \textcolor{blue}{49.8} & \multicolumn{1}{c|}{\textcolor{blue}{72.0}} & \textcolor{blue}{60.7} \\
PURLS\cite{zhu2024part} & CVPR2024 & \multicolumn{1}{c|}{79.2} & 41.0 & \multicolumn{1}{c|}{\textcolor{blue}{72.0}} & 52.0 \\
SA-DVAE\cite{li2025sa} & ECCV2024 & \multicolumn{1}{c|}{82.4} & 41.4 & \multicolumn{1}{c|}{68.8} & 46.1 \\ \hline
Ours$^{\dag}$ & \textbf{\textbackslash{}} & 
\multicolumn{1}{c|}{84.2} & 
52.6 & 
\multicolumn{1}{c|}{71.2} & 
61.9 \\ 
\rowcolor{gray!30}
\textbf{Ours} & \textbf{\textbackslash{}} & 
\multicolumn{1}{c|}{\textbf{\textcolor{red}{86.9}}$_{\textcolor{red}{\uparrow 3.6}}$} & 
\textbf{\textcolor{red}{57.2}}$_{\textcolor{red}{\uparrow 7.4}}$ & 
\multicolumn{1}{c|}{\textbf{\textcolor{red}{74.4}}$_{\textcolor{red}{\uparrow 2.4}}$} & 
\textbf{\textcolor{red}{62.5}}$_{\textcolor{red}{\uparrow 1.8}}$ \\ \hline
\end{tabular}
}
\label{tab: zsl results}
\vspace{-15pt}
\end{table}

\begin{table*}[t]
\scriptsize
\renewcommand\arraystretch{1.0}
\centering
  \caption{Generalized Zero-Shot Learning Results. The highest values are highlighted in red, and the second-highest values (from other works) are marked in blue. H represents the harmonic mean. The result analysis is presented in Section \ref{sec: sota}.}
  \vspace{-5pt}
   \setlength\tabcolsep{5.5pt} 
{
\begin{tabular}{c|c|c>{\columncolor{lightgray}}c>{\columncolor{gray!30}}c|c>{\columncolor{lightgray}}c>{\columncolor{gray!30}}c|c>{\columncolor{lightgray}}c>{\columncolor{gray!30}}c|c>{\columncolor{lightgray}}c>{\columncolor{gray!30}}c}
\hline
\multirow{2}{*}{Methods} & \multirow{2}{*}{Venue} & \multicolumn{3}{c|}{NTU-60 (55/5 split)} & \multicolumn{3}{c|}{NTU-60 (48/12 split)} & \multicolumn{3}{c|}{NTU-120 (110/10 split)} & \multicolumn{3}{c}{NTU-120 (96/24 split)} \\ \cline{3-14} 
 &  & Seen & \cellcolor{lightgray}Unseen & \cellcolor{gray!30}H & Seen & \cellcolor{lightgray}Unseen & \cellcolor{gray!30}H & Seen & \cellcolor{lightgray}Unseen & \cellcolor{gray!30}H & Seen & \cellcolor{lightgray}Unseen & \cellcolor{gray!30}H \\ \hline
ReViSE\cite{hubert2017learning} & ICCV 2017 & 74.2 & 34.7 & 29.2 & 62.4 & 20.8 & 31.2 & 48.7 & 44.8 & 46.7 & 49.7 & 25.1 & 33.3 \\
JPoSE\cite{wray2019fine} & ICCV 2019 & 64.4 & 50.3 & 56.5 & 60.5 & 20.6 & 30.8 & 47.7 & 46.4 & 47.0 & 38.6 & 22.8 & 28.7 \\
CADA-VAE\cite{schonfeld2019generalized} & CVPR 2019 & 69.4 & 61.8 & 65.4 & 51.3 & 27.0 & 35.4 & 47.2 & 19.8 & 48.4 & 41.1 & 34.1 & 37.3 \\
SynSE\cite{gupta2021syntactically} & ICIP2021 & 61.3 & 56.9 & 59.0 & 52.2 & 27.9 & 36.3 & 52.5 & 57.6 & 54.9 & 56.4 & 32.2 & 41.0 \\
STAR\cite{chen2024fine} & ACMM2024 & 69.0 & 69.9 & \textcolor{blue}{69.4} & 62.7 & 37.0 & 46.6 & 59.9 & 52.7 & 56.1 & 51.2 & 36.9 & 42.9 \\
GZSSAR*\cite{li2023multi} & ICIG2023 & 66.8 & 70.7 & 68.7 & 54.8 & \textcolor{blue}{41.4} & \textcolor{blue}{47.1} & 58.1 & 57.8 & 58.0 & 59.2 & \textcolor{blue}{45.9} & \textcolor{blue}{51.7} \\
SA-DVAE\cite{li2025sa} & ECCV2024 & 62.3 & \textcolor{blue}{70.8} & 66.3 & 50.2 & 36.9 & 42.6 & 61.1 & \textcolor{blue}{59.8} & \textcolor{blue}{60.4} & 58.8 & 35.8 & 44.5 \\ \hline
Ours$^{\dag}$ & \textbackslash{} & 76.4 & 61.9 & 68.4 & 57.4 & 43.5 & 49.5 & 55.7 & 66.8 & 60.7 & 58.7 & 48.3 & 53.0  \\ 
\textbf{Ours} & \textbackslash{} & 77.0 & \textbf{\textcolor{red}{74.5}}$_{\textcolor{red}{\uparrow 3.7}}$ & 
\textbf{\textcolor{red}{75.7}}$_{\textcolor{red}{\uparrow 6.3}}$ & 56.2 & 
\textbf{\textcolor{red}{48.6}}$_{\textcolor{red}{\uparrow 7.2}}$ & 
\textbf{\textcolor{red}{52.1}}$_{\textcolor{red}{\uparrow 5.0}}$ & 59.2 & 
\textbf{\textcolor{red}{67.9}}$_{\textcolor{red}{\uparrow 8.1}}$ & 
\textbf{\textcolor{red}{63.3}}$_{\textcolor{red}{\uparrow 2.9}}$ & 57.8 & 
\textbf{\textcolor{red}{51.9}}$_{\textcolor{red}{\uparrow 6.0}}$ & 
\textbf{\textcolor{red}{54.7}}$_{\textcolor{red}{\uparrow 3.0}}$ \\ \hline
\end{tabular}
}
\label{tab: gzsl results}
\vspace{-15pt}
\end{table*}
\section{Experiments}

\subsection{Implementation Details}
\label{sec:implementation}
Our work mainly follows \cite{gupta2021syntactically} for data preprocessing. The data split strategy follows \cite{gupta2021syntactically, li2023multi}. In ZSL, a 55/5 split means training on 55 seen classes and testing on 5 unseen classes. In Generalized Zero-Shot Learning (GZSL), training remains the same, but testing includes both the 55 seen and 5 unseen classes.
We adopt Shift-GCN \cite{cheng2020skeleton} as the skeleton extractor. Meanwhile, GPT-4 \cite{achiam2023gpt} is utilized to generate action descriptions. More settings can be found in \textcolor{magenta}{Appendix B}, and please refer to \textcolor{magenta}{Appendix C} for additional experiments, including results on the PKU-MMD dataset \cite{liu2017pku}.

\subsection{Comparisons with State-of-the-Art Methods}
\label{sec: sota}
To evaluate the effectiveness of our approach, we compare it with state-of-the-art methods under both ZSL and GZSL settings. The results on NTU-60 \cite{shahroudy2016ntu} and NTU-120 \cite{liu2019ntu} datasets, following established split protocols \cite{gupta2021syntactically, li2023multi}, are summarized in Tables \ref{tab: zsl results} and \ref{tab: gzsl results}. Our model consistently achieves the highest accuracy in ZSL, demonstrating strong generalization to unseen actions. In GZSL, it outperforms existing methods, leading to the highest harmonic mean score\cite{gupta2021syntactically} (H-score, $H = \frac{2 \times \text{Seen} \times \text{Unseen}}{\text{Seen} + \text{Unseen}}
$), highlighting the effectiveness of FS-VAE in learning a semantic-enhanced and well-calibrated representation. Notably, we focus on unseen accuracy and the H-score as they best reflect generalization. Unseen accuracy measures the model’s ability to classify novel actions, while H-score balances seen and unseen performance to prevent biased predictions.


\subsection{Ablation Study}

\textbf{Influence of Different Modules.}
Table \ref{tab: Different Modules} highlights the impact of each key component in our framework, including the Semantic-based Descriptions (SD), Frequency-enhanced Module (FM), and Calibrated Loss (CL). Adding SD improves the accuracy of the baseline to 85.4\% in the NTU-60 dataset (e.g., 55/5 split), emphasizing the importance of semantic enrichment. Similarly, integrating FM or CL individually achieves 85.8\% and 84.4\%, respectively, demonstrating their individual contributions to frequency-specific feature learning and robust alignment. Combining all three components leads to the highest performance of 86.9\%, which confirms their complementary effects in addressing the ZSSAR challenges.
\begin{table}[]
\scriptsize
\renewcommand\arraystretch{1.0}
\centering
\caption{Influence of different modules. Semantic-based action Descriptions (SD), Frequency-enhanced Module (FM), Calibrated Loss (CL).}
\vspace{-5pt}
\setlength\tabcolsep{6.0pt} 
{
\begin{tabular}{ccccccclll}
\cline{1-7}
\multicolumn{3}{c|}{Modules} & \multicolumn{2}{c|}{NTU-60 (ACC,\%)} & \multicolumn{2}{c}{NTU-120 (ACC,\%)} &  &  &  \\ \cline{1-7}
SD & FM & \multicolumn{1}{c|}{CL} & 55/5 split & \multicolumn{1}{c|}{48/12 split} & 110/10 split & 96/24 split &  &  &  \\ \cline{1-7}
\ding{55} & \ding{55} & \multicolumn{1}{c|}{\ding{55}} & 83.3 & \multicolumn{1}{c|}{49.8} & 72.0 & 60.7 &  &  &  \\
\ding{51} & \ding{55} & \multicolumn{1}{c|}{\ding{55}} & 85.4 & \multicolumn{1}{c|}{52.7} & 73.0 & 61.3 &  &  &  \\
\ding{55} & \ding{51} & \multicolumn{1}{c|}{\ding{55}} & 85.8 & \multicolumn{1}{c|}{53.1} & 74.0 & 61.8 &  &  &  \\
\ding{55} & \ding{55} & \multicolumn{1}{c|}{\ding{51}} & 84.4 & \multicolumn{1}{c|}{54.1} & 72.8 & 60.0 &  &  &  \\
\ding{51} & \ding{51} & \multicolumn{1}{c|}{\ding{51}} & \textbf{86.9} & \multicolumn{1}{c|}{\textbf{57.2}} & \textbf{74.4} & \textbf{62.5} &  &  &  \\ \cline{1-7}
\end{tabular}
}
\label{tab: Different Modules}
\vspace{-15pt}
\end{table}

\begin{table}[]
\scriptsize
\renewcommand\arraystretch{1.0}
\centering
\caption{Influence of different text descriptions. Action Label (AL), Local action Description (LD), and Global action Description (GD).}
\vspace{-5pt}
\setlength\tabcolsep{6.0pt} 
{
\begin{tabular}{ccccccclll}
\cline{1-7}
\multicolumn{3}{c|}{Descriptions} & \multicolumn{2}{c|}{NTU-60 (ACC,\%)} & \multicolumn{2}{c}{NTU-120 (ACC,\%)} &  &  &  \\ \cline{1-7}
AL & LD & \multicolumn{1}{c|}{GD} & 55/5 split & \multicolumn{1}{c|}{48/12 split} & 110/10 split & 96/24 split &  &  &  \\ \cline{1-7}
\ding{51} & \ding{55} & \multicolumn{1}{c|}{\ding{55}} & 81.9 & \multicolumn{1}{c|}{38.7} & 70.3 & 47.3 &  &  &  \\
\ding{55} & \ding{51} & \multicolumn{1}{c|}{\ding{55}} & 79.3 & \multicolumn{1}{c|}{43.8} & 54.5 & 45.7 &  &  &  \\
\ding{55} & \ding{55} & \multicolumn{1}{c|}{\ding{51}} & 82.0 & \multicolumn{1}{c|}{48.6} & 64.7 & 59.2 &  &  &  \\
\ding{55} & \ding{51} & \multicolumn{1}{c|}{\ding{51}} & 83.7 & \multicolumn{1}{c|}{54.1} & 57.5 & 46.7 &  &  &  \\
\ding{51} & \ding{51} & \multicolumn{1}{c|}{\ding{51}} & \textbf{86.9} & \multicolumn{1}{c|}{\textbf{57.2}} & \textbf{74.4} & \textbf{62.5} &  &  &  \\ \cline{1-7}
\end{tabular}
}

\label{tab: text descriptions}
\vspace{-20pt}
\end{table}

\textbf{Influence of Different Text Descriptions.}
Table \ref{tab: text descriptions} presents the influence of different text descriptions, including Action Label (AL), Local action Description (LD), and Global action Description (GD). Using AL alone achieves an accuracy of 81.9\% on the NTU-60 dataset (55/5 split), showing its fundamental role in providing basic semantic information. Incorporating LD or GD results in 79.3\% and 82.0\%, respectively, suggesting that GD contributes more to performance as it provides more comprehensive semantics of the overall action. Combining LD and GD boosts the performance to 83.7\%, highlighting the synergy between these two semantic-aware features. Integrating all three components achieves the highest accuracy of 86.9\%, demonstrating their complementary contributions to enriching semantic embeddings in ZSSAR.

\begin{figure*}[t]
    \centering
        \begin{minipage}{\textwidth}
            \centering
            \begin{subfigure}[t]{0.49\linewidth} 
                \centering
                \includegraphics[width=\linewidth]{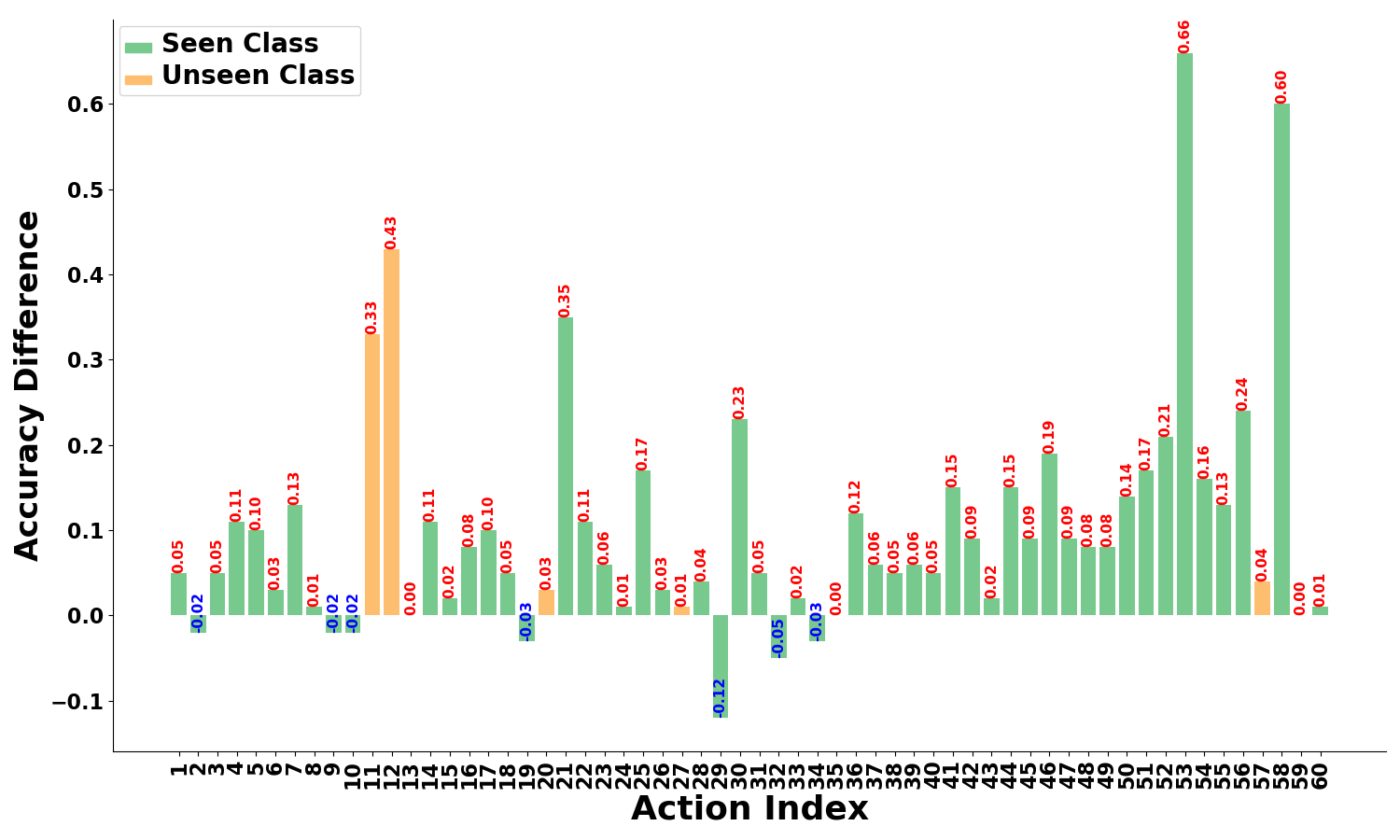}
                \caption{Accuracy difference}
                \label{fig:all_diff}
            \end{subfigure}
            \hfill
            \begin{subfigure}[t]{0.25\linewidth} 
                \centering
                \includegraphics[width=\linewidth]{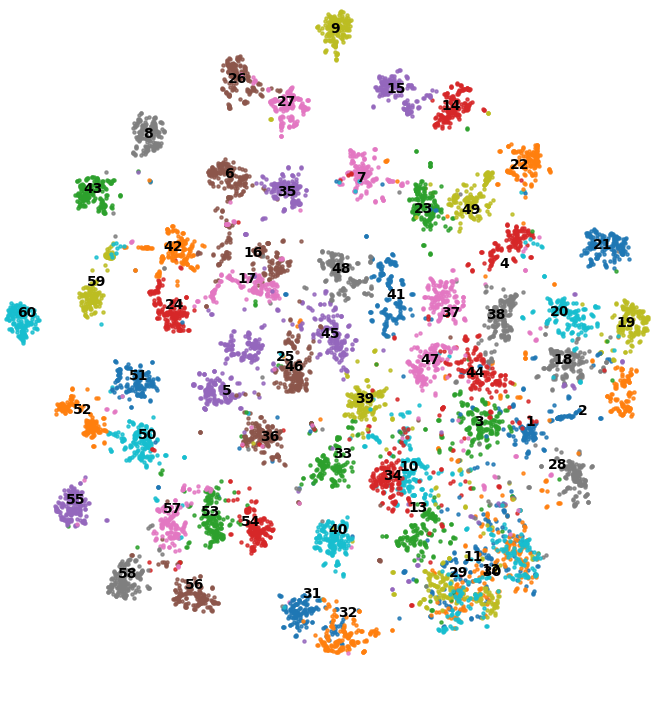}
                \caption{t-SNE for baseline \cite{li2023multi}}
                \label{fig:tsne_baseline}
            \end{subfigure}
            \hfill
            \begin{subfigure}[t]{0.25\linewidth} 
                \centering
                \includegraphics[width=\linewidth]{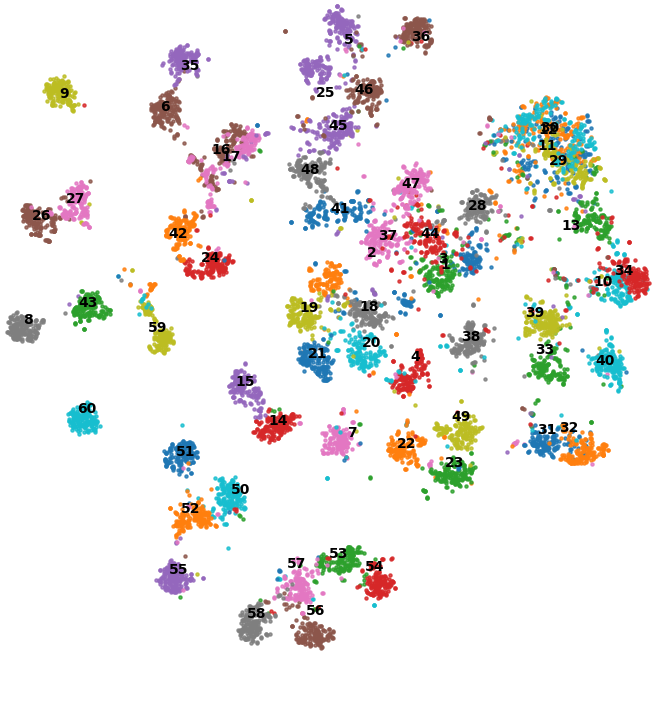}
                \caption{t-SNE for our method}
                \label{fig:tsne_ours}
            \end{subfigure}
        \end{minipage}
    \vspace{-5pt}
    \caption{(a) shows the accuracy difference for seen-unseen actions compared to baseline \cite{li2023multi} under the NTU-60 55/5 split, where the outperforming accuracies are marked in red, and others are in blue.
    (b) and (c) depict the t-SNE visualizations, the corresponding action indices (listed in \textcolor{magenta}{Appendix G}) are labeled in the clusters. Best viewed by zooming in.}
    \label{fig:main}
    \vspace{-15pt}
\end{figure*}

\textbf{Influence of \(\varphi\) and \(b\) in Frequency Enhanced Module.}  
In Fig. \ref{fig:fig4}(a)-(b), we evaluate the impact of various hyperparameter settings of our frequency-enhanced module. The analysis reveals that \(\varphi\) and \(b\) play a crucial role in determining the overall accuracy of our method. The optimal configuration is achieved with \(\varphi = 35\) and \(b = 30\). These values effectively balance feature enhancement of both global action structures and fine-grained motion details. 
\(\varphi\) controls the separation between low- and high-frequency components, ensuring that structural semantics are preserved while capturing subtle motion variations. Meanwhile, \(b\) determines the intensity of enhancement, preventing the amplification of noise while enhancing the model's ability of discriminative representation learning. 


\begin{figure}[ht!]
  \centering
  \begin{subfigure}[b]{0.49\linewidth}
    \centering
    \includegraphics[width=\linewidth]{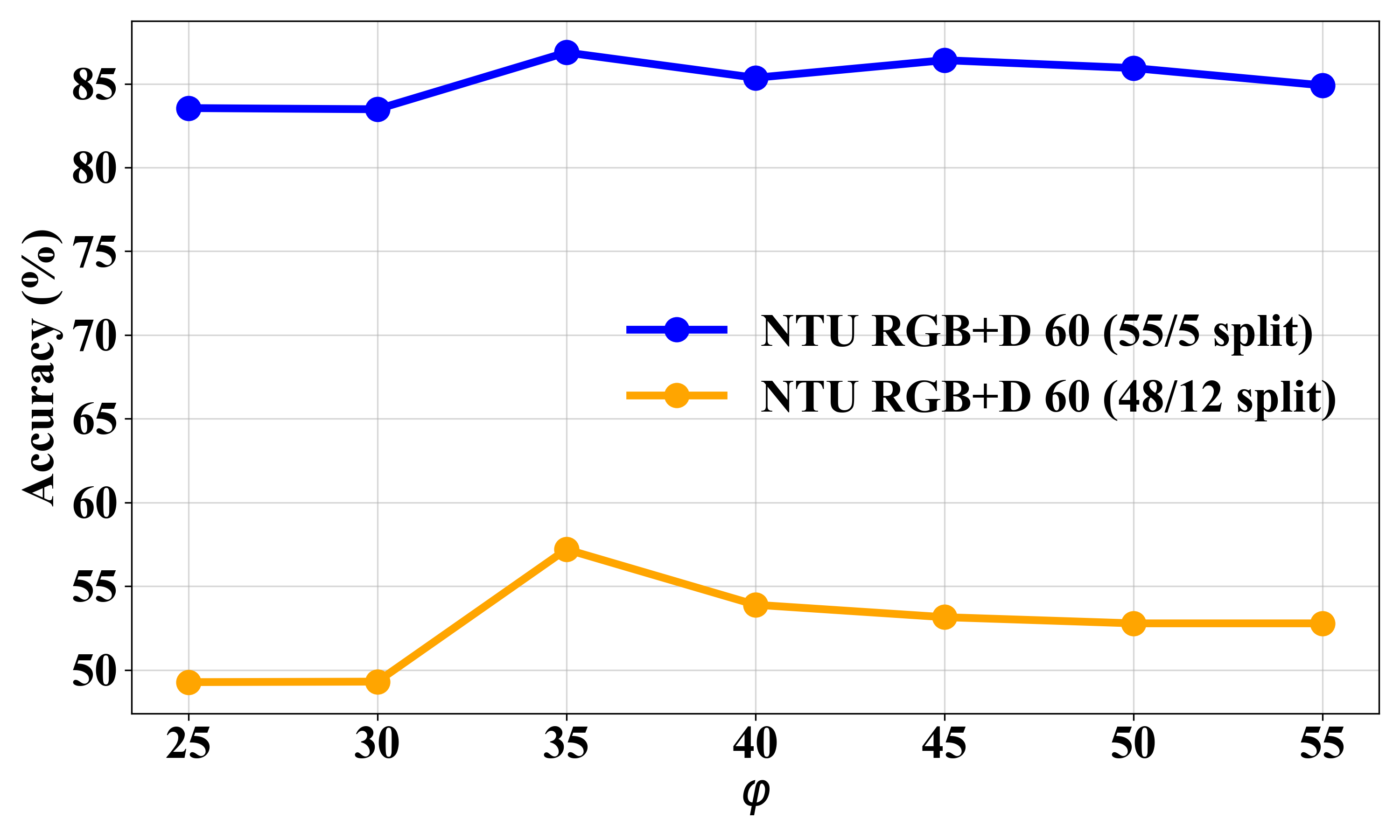}
    \caption{Influence of $\varphi$ on NTU-60 split}
    \label{fig:fig4_phi}
  \end{subfigure}
  \hfill
  \begin{subfigure}[b]{0.49\linewidth}
    \centering
    \includegraphics[width=\linewidth]{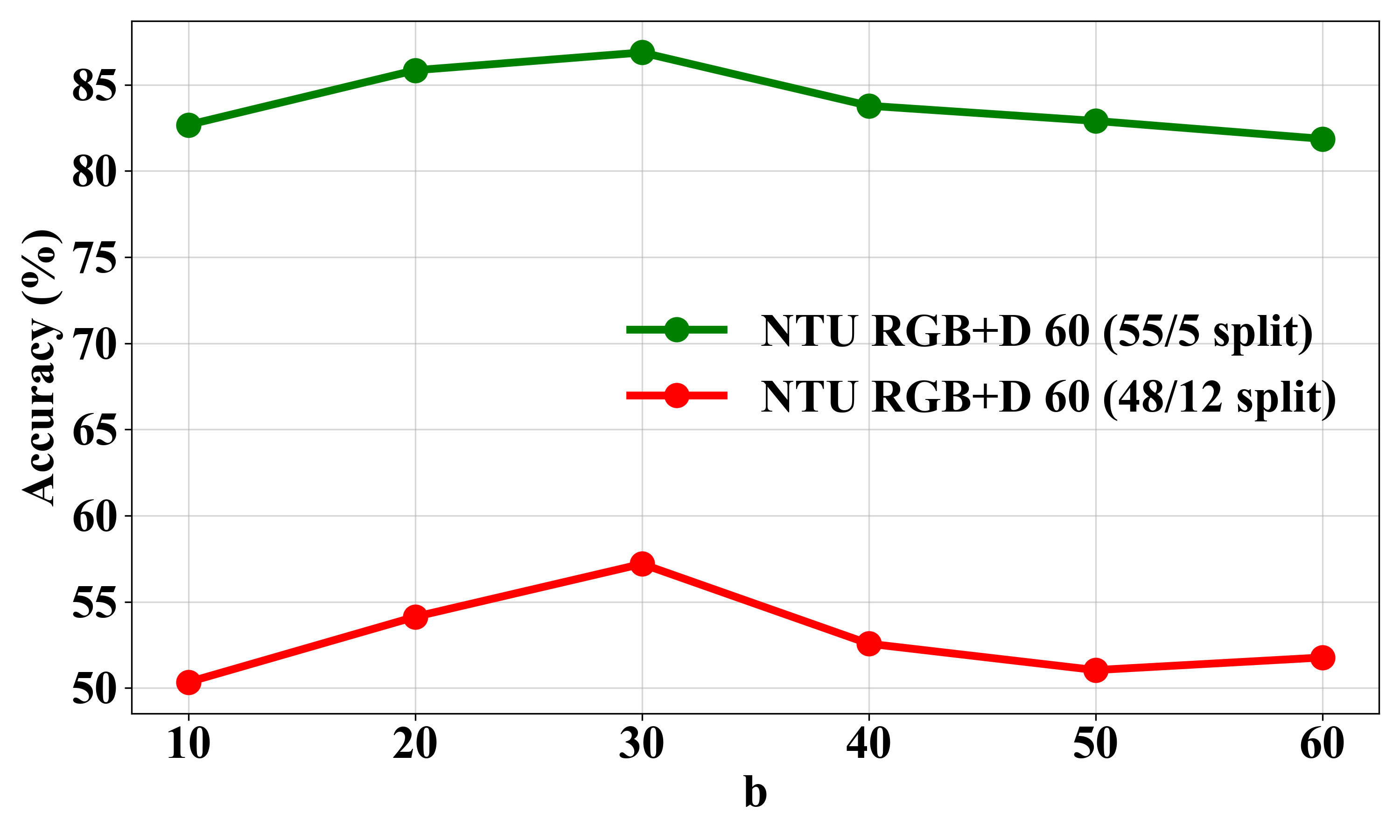}
    \caption{Influence of $b$ on NTU-60 split}
    \label{fig:fig4_b}
  \end{subfigure}
     \vspace{-5pt}
  \caption{Influence of $\varphi$ and $b$ in frequency enhanced module.}
  \label{fig:fig4}
  \vspace{-15pt}
\end{figure}

\textbf{Influence of $\alpha$ and $\lambda$ for Calibrated Loss.}
Recall that \(\alpha\) balances the reconstruction loss and alignment loss, playing a crucial role in ensuring that both losses contribute effectively to the overall objective. Meanwhile, $\lambda$ controls the sensitivity of the alignment loss, with smaller values making it more responsive to misalignments. Specifically, a small 
$\lambda$ places greater emphasis on larger misalignments compared to a large 
$\lambda$, which reduces this sensitivity, making the alignment loss less responsive to smaller misalignments. 
Notably, the alignment loss retains its symmetric property regardless of the choice of \(\lambda\).
As shown in Table \ref{tab:alpha_lambda_results}, the analysis reveals that the optimal combination of parameters is $\alpha = 0.1$ and $\lambda = 100$ to yield the best performance.

\textbf{Impact of Removing Frequency Adjustment.}
In Tables \ref{tab: zsl results} and \ref{tab: gzsl results} (denoted by `Ours$^\dag$'), we also evaluate the impact of removing explicit frequency adjustment and replacing it with purely learnable frequency weight:
$\mathbf{S} \leftarrow \mathbf{S} \cdot g(i)$,
where \( g(i) =  w_i\). \(w_i\) is the learnable weight applied directly to all frequency components.  
The performance drop highlights two key issues: (1) Without explicit frequency adjustment, the model fails to balance global structural patterns and fine-grained details, leading to weaker semantic alignment and increased sensitivity to noise. (2) Explicit frequency scaling provides a prior-informed enhancement, whereas purely learnable weights rely solely on data-driven optimization.
This often results in inconsistent frequency adjustments across different training samples, leading to overfitting to seen categories and ineffective generalization to unseen actions.

\begin{table}[t]
\vspace{-5pt}
\scriptsize
\renewcommand\arraystretch{1.1}
\centering
\caption{Influence of $\alpha$ and $\lambda$ in calibrated loss.}
\vspace{-5pt}
\setlength\tabcolsep{10pt}
{
\begin{tabular}{cllcccclll}
\cline{1-7}
\multicolumn{3}{c|}{\multirow{2}{*}{$\alpha$}} & \multicolumn{2}{c|}{NTU-60 (ACC,\%)} & \multicolumn{2}{c}{NTU-120 (ACC,\%)} &  &  &  \\ \cline{4-7}
\multicolumn{3}{c|}{} & 55/5 split & \multicolumn{1}{c|}{48/12 split} & 110/10 split & 96/24 split &  &  &  \\ \cline{1-7}
\multicolumn{3}{c|}{0.01} & 81.9 & \multicolumn{1}{c|}{51.8} & 74.1 & 61.8 &  &  &  \\
\multicolumn{3}{c|}{0.05} & 84.0 & \multicolumn{1}{c|}{49.4} & 72.8 & 60.0 &  &  &  \\
\multicolumn{3}{c|}{0.1} & \textbf{86.9} & \multicolumn{1}{c|}{\textbf{57.2}} & \textbf{74.4} & \textbf{62.5} &  &  &  \\
\multicolumn{3}{c|}{0.5} & 85.9 & \multicolumn{1}{c|}{53.7} & 74.0 & 61.8 &  &  &  \\
\multicolumn{3}{c|}{1} & 83.2 & \multicolumn{1}{c|}{51.8} & 72.8 & 61.9 &  &  &  \\
\multicolumn{3}{c|}{2} & 81.0 & \multicolumn{1}{c|}{50.3} & 71.3 & 60.4 &  &  &  \\ \cline{1-7}
\multicolumn{10}{c}{} \\ 
\cline{1-7}
\multicolumn{3}{c|}{\multirow{2}{*}{$\lambda$}} & \multicolumn{2}{c|}{NTU-60 (ACC,\%)} & \multicolumn{2}{c}{NTU-120 (ACC,\%)} &  &  &  \\ \cline{4-7}
\multicolumn{3}{c|}{} & 55/5 split & \multicolumn{1}{c|}{48/12 split} & 110/10 split & 96/24 split &  &  &  \\ \cline{1-7}
\multicolumn{3}{c|}{70} & 84.1 & \multicolumn{1}{c|}{51.9} & 71.7 & 60.7 &  &  &  \\
\multicolumn{3}{c|}{80} & 84.9 & \multicolumn{1}{c|}{53.1} & 73.0 & 61.9 &  &  &  \\
\multicolumn{3}{c|}{90} & 86.0 & \multicolumn{1}{c|}{53.8} & 73.0 & 60.3 &  &  &  \\
\multicolumn{3}{c|}{100} & \textbf{86.9} & \multicolumn{1}{c|}{\textbf{57.2}} & \textbf{74.4} & \textbf{62.5} &  &  &  \\
\multicolumn{3}{c|}{110} & 86.2 & \multicolumn{1}{c|}{53.8} & 74.0 & 61.9 &  &  &  \\
\multicolumn{3}{c|}{120} & 85.7 & \multicolumn{1}{c|}{52.4} & 74.2 & 60.2 &  &  &  \\ \cline{1-7}
\end{tabular}
}
\label{tab:alpha_lambda_results}
\vspace{-15pt}
\end{table}



\subsection{Qualitative Analysis}

We present the accuracy difference results compared to the baseline method for the NTU-60 55/5 split in Fig. \ref{fig:all_diff}. Our approach consistently outperforms the baseline across both seen and unseen actions. The results highlight our method's effectiveness in not only enhancing overall accuracy but also improving recognition across most discriminative actions. Notably, FS-VAE surpasses the baseline in actions that require fine-grained motion understanding, such as “reading” and “writing” in unseen classes (orange) and most of the actions in seen classes (green). Additionally, Fig. \ref{fig:tsne_baseline} and \ref{fig:tsne_ours} present t-SNE \cite{van2008visualizing} visualization examples of NTU-60 dataset under 55/5 split. The results illustrate that our method improves the visual and semantic alignment (e.g., better inter-class separation between a pair of skeletal-similar actions, such as “reading” and “type on a keyboard”). Furthermore, it produces a tighter and semantically structured embedding space (e.g., stronger intra-class cohesion of “reading” and “pushing”).

\vspace{-5pt}
\section{Conclusion}
\vspace{-5pt}
We introduce a novel framework for zero-shot skeleton-based action recognition (ZSSAR) that combines frequency-enhanced modeling with a calibrated alignment mechanism. The frequency-enhanced module leverages DCT to capture fine-grained details and preserves global patterns. The semantic-based action description enriches feature embeddings, while the calibrated cross-alignment loss dynamically addresses modality gaps and ambiguities. Extensive evaluations on the benchmarks demonstrate the state-of-the-art performance of our approach in recognizing unseen actions. This work establishes a robust ZSSAR framework, paving the way for future advances in frequency-aware action recognition.

{\small
\bibliographystyle{ieeenat_fullname}
\bibliography{main}
}


\appendix
\section{Appendix}
The supplementary material is organized into the following sections:
\begin{itemize}
    \item Section \ref{app:experiment settings}: \textbf{More experimental settings.}
    (i) Datasets introduction (NTU-60, NTU-120, PKU-MMD); (ii) training strategy; (iii) parameter settings.
    
    \item Section \ref{app:experiments}: \textbf{Additional experiments.}
    (i) Results on PKU-MMD; (ii) results on different text feature extractors.
    
    \item Section \ref{app: semantic}: \textbf{Semantic-based action descriptions.}
    (i) Prompting examples; (ii) description examples.
    
    \item Section \ref{app:cali_explain}: \textbf{Calibrated alignment loss analysis.}
    (i) Calibrated alignment loss explanation; (ii) extra ablation study for calibrated alignment loss.
    
    \item Section \ref{app:dct_proof}: \textbf{Frequency-based skeleton representation analysis.}
    (i) Frequency domain representation and energy preservation proof; (ii) semantic integrity with frequency adjustment; (iii) frequency-based enhancement mechanism; (iv) energy redistribution derivation; (v) illustration example of frequency enhanced method; (vi) codes.
    
    \item Section \ref{app:freq_discussion}: \textbf{Justification for choosing DCT.}

    \item Section \ref{app:index}: \textbf{NTU-60 dataset action index.}
\end{itemize}
\section{More Experiments Settings}
\label{app:experiment settings}

\subsection{Datasets}
\textbf{NTU RGB+D 60 \cite{shahroudy2016ntu}.} The NTU-60 dataset is one of the most popular large-scale datasets designed for the analysis of 3D human actions. It comprises 56,880 human action sequences captured by three Kinect-V2 cameras, covering 60 distinct action classes. In this work, we use only the skeleton data. Each skeleton sequence consists of up to two skeletons per frame, with each skeleton containing 25 joints. In this paper, two seen/unseen splits are employed, following prior work \cite{gupta2021syntactically}: 55 seen classes and 5 unseen classes, and 48 seen classes and 12 unseen classes. The unseen classes are randomly selected, maintaining consistency with previous studies. 

\textbf{NTU RGB+D 120 \cite{liu2019ntu}.} The NTU-120 dataset is an extended version of NTU-60. It includes 114,480 action sequences performed by 106 subjects from 155 distinct viewpoints, spanning 120 action classes. These 120 classes build upon the original 60 classes in NTU-60, offering a broader range of human actions. For zero-shot learning, the dataset adopts seen/unseen splits of 110 seen classes and 10 unseen classes, and 96 seen classes and 24 unseen classes, consistent with the splits defined in \cite{gupta2021syntactically}.

\textbf{PKU-MMD \cite{liu2017pku}.} The PKU-MMD dataset is a large-scale benchmark for multimodal action recognition, providing both 3D skeleton sequences and RGB+D recordings. It consists of 66 subjects and 51 classes. We conduct the experiments on Phase I following the protocols from \cite{li2023multi, li2025sa} and the skeleton features provided by \cite{li2025sa} for a fair comparison (skeleton features are generated by ST-GCN\cite{yan2018spatial}, 46/5 split settings, 46 seen classes and 5 unseen classes).  

\begin{table}[]
\scriptsize
\renewcommand\arraystretch{1.1}
\centering
  \caption{Zero-Shot Learning (ZSL) and Generalized Zero-Shot Learning (GZSL) results on PKU-MMD (46/5 split).}
  \vspace{-5pt}
   \setlength\tabcolsep{4.0pt} 
{
\begin{tabular}{c|c|c|c>{\columncolor{lightgray}}c>{\columncolor{gray!30}}c}
\hline
\multirow{2}{*}{Methods} & \multirow{2}{*}{Venue} & \multirow{2}{*}{ZSL (ACC,\%)} & \multicolumn{3}{c}{GZSL (ACC,\%)}\\ \cline{4-6} 
 &  &  & Seen & Unseen & H \\ \hline
ReViSE\cite{hubert2017learning} & ICCV2017 & 59.3 & 60.9 & 42.2 & 49.8 \\
JPoSE\cite{wray2019fine} & ICCV2019 & 57.2 & 60.3 & 45.2 & 51.6 \\
CADA-VAE\cite{schonfeld2019generalized} & CVPR2019 & 60.7 & 63.2 & 35.9 & 45.8 \\
SynSE\cite{gupta2021syntactically} & ICIP2021 & 53.9 & 63.1 & 40.7 & 49.5 \\
SMIE\cite{zhou2023zero} & ACMM2023 & 60.8 & - & - & - \\
SA-DVAE\cite{li2025sa} & ECCV2024 & \textcolor{blue}{66.5} & 58.5 & \textcolor{blue}{51.4} & \textcolor{blue}{54.7} \\
\textbf{Ours} & \textbf{\textbackslash{}} & \textbf{\textcolor{red}{71.2}}$_{\textcolor{red}{\uparrow 4.7}}$ & 
64.3 & \textbf{\textcolor{red}{54.5}}$_{\textcolor{red}{\uparrow 3.1}}$ & \textbf{\textcolor{red}{59.0}}$_{\textcolor{red}{\uparrow 4.3}}$ \\ \hline
\end{tabular}
}
\label{tab: pku results}
\end{table}
\begin{table}[t]
\scriptsize
\renewcommand\arraystretch{1.1}
\centering
\caption{Comparisons of different text feature extractors in ZSL.}
\vspace{-5pt}
\setlength\tabcolsep{6.0pt}
{
\begin{tabular}{c|cc|cc}
\hline
\multirow{2}{*}{Model} & \multicolumn{2}{c|}{NTU-60 (ACC,\%)} & \multicolumn{2}{c}{NTU-120 (ACC,\%)} \\ \cline{2-5} 
 & 55/5 split & 48/12 split & 110/10 split & 96/24 split \\ \hline
ViT-B/16 & 84.2 & 49.4 & 72.7 & 60.2 \\
ViT-B/32 & \textbf{86.9} & \textbf{57.2} & \textbf{74.4} & \textbf{62.5} \\ \hline
\end{tabular}
}
\label{tab: zsl3}
\vspace{-15pt}
\end{table}

\begin{table*}[htbp]
\scriptsize
\renewcommand\arraystretch{1.2}
\centering
\caption{Comparisons of different text feature extractors in GZSL.}
\vspace{-5pt}
\setlength\tabcolsep{6.0pt}
{
\begin{tabular}{c|ccc|ccc|ccc|ccc}
\hline
\multirow{2}{*}{Model} & \multicolumn{3}{c|}{NTU-60 (55/5 split)} & \multicolumn{3}{c|}{NTU-60 (48/12 split)} & \multicolumn{3}{c|}{NTU-120 (110/10 split)} & \multicolumn{3}{c}{NTU-120 (96/24 split)} \\ \cline{2-13} 
 & Seen & Unseen & H & Seen & Unseen & H & Seen & Unseen & H & Seen & Unseen & H \\ \hline
ViT-B/16 & 65.1 & 71.0 & 67.9 & 61.0& 39.4 & 47.9 & 55.5 & \textbf{68.9} & 61.4 & 56.6 & 47.7 & 52.6 \\
ViT-B/32 & 77.0& \textbf{74.5} & \textbf{75.7} & 56.2 & \textbf{48.6} & \textbf{52.1} & 59.2& 67.9 & \textbf{63.3} & 57.8& \textbf{51.9} & \textbf{54.7} \\ \hline
\end{tabular}
}
\label{tab: gzsl2}
\vspace{-5pt}
\end{table*}
\subsection{Training Strategy}
The training phase follows the same processing procedure as \cite{li2023multi}, which is systematically organized into four stages: training the skeleton feature extractor to capture spatio-temporal dependencies, optimizing the generative cross-modal alignment module to bridge the skeleton and semantic features, training the unseen class classifier for generalization, and the seen-unseen classification gate for accurate category differentiation.

\subsection{Parameter Settings}
Table \ref{tab:Implementation} shows the parameter settings of our method, including the parameters applied during all the training stages mentioned in the main paper and \cite{li2023multi}.

\section{More Experiments}
\label{app:experiments}
\textbf{Results on PKU-MMD.} 
Table \ref{tab: pku results} presents the ZSL and GZSL performance on the PKU-MMD dataset under the 46/5 split settings \cite{li2025sa}. Our approach consistently outperforms prior methods in both ZSL and GZSL settings, demonstrating its effectiveness in recognizing unseen actions while maintaining strong generalization.

\textbf{Comparisons of Different Text Feature Extractors.}
We evaluate two CLIP-based text encoders, ViT-B/16 and ViT-B/32, for ZSSAR and GZSSAR tasks on NTU-60 and NTU-120 datasets. As shown in Table~\ref{tab: zsl3}, ViT-B/32 achieves higher ZSL accuracies in all splits, e.g., 86.9\% vs. 84.2\% on the NTU-60 55/5 split. For GZSSAR in Table~\ref{tab: gzsl2}, ViT-B/32 also outperforms ViT-B/16 in harmonic mean (H-score), e.g., 75.7\% vs. 67.9\% on the NTU-60 55/5 split. Based on these results, we use ViT-B/32 as the text feature extractor in subsequent experiments.

\section{Semantic-based Action Descriptions}
\label{app: semantic}
\textbf{Global Action Description Prompting Examples.}
\textit{"Describe the action of [ACTION NAME] by summarizing its overall motion pattern and intent. Focus on the key movements that define the action as a whole. Avoid excessive details about specific joints but ensure the description captures how the action is performed in a natural way. For example, describe how objects are manipulated, how body posture changes, or the general sequence of motion from start to finish."}

\textbf{Local Action Description Prompting Examples.}
\textit{"Describe the action of [ACTION NAME] by detailing the precise movements of the hands, arms, or other involved body parts. Provide a step-by-step breakdown of how the action is executed at a fine-grained level, emphasizing joint motion, hand positioning, and transitions. Ensure the description remains human-readable and avoids overly technical terminology."}

\textbf{Description Examples.}
Table~\ref{tab:semantic description} illustrates how our method refines action descriptions by incorporating both \textcolor{globalcolor}{global} and \textcolor{localcolor}{local} semantic components. Compared to the baseline\cite{li2023multi}, which provides a vague summary, our approach explicitly decomposes actions into structured representations. 

For example, in the action ``drinking water ", the baseline only mentions the ingestion process, whereas our Global action Description (GD) highlights the sequential motion of ``grasping an object, raising it to the head, and simulating a drinking motion", capturing the structural essence of the action. Meanwhile, Local action Description (LD) provides finer details, such as ``moving the fist up to the head and looking slightly downward", which are critical for distinguishing similar actions like ``eating".

Similarly, for ``Brushing Teeth", the baseline merely describes the purpose of the action (``to clean teeth with a brush"), but GD focuses on the characteristic motion of ``moving a toothbrush back and forth", while the LD refines it further by specifying ``hand movement towards the head followed by wrist tremble". This level of granularity ensures better alignment between textual descriptions and skeleton-based representations.

These examples demonstrate that our description method not only improves semantic precision, which is crucial for robust skeleton-based action recognition. By explicitly decomposing actions into structured representations that encompass both global motion patterns and localized details, the model gains a more comprehensive understanding of action semantics. This enriched textual description provides a stronger supervision signal for aligning skeleton features with semantic embeddings, thereby reducing ambiguities in action recognition. 

\section{Analysis of Calibrated Alignment Loss}
\label{app:cali_explain}
\subsection{Calibrated Loss Explanation}
In this section, we break down the loss function to analyze how the calibrated alignment loss operates. Without loss of generality, consider a multi-class classification problem with three classes: Class 1, Class 2, and Class 3. Each class is associated with a ground truth distribution, denoted as $P_1$, $P_2$, and $P_3$. 
Assume we collect a dataset as follows: 1) $S_1$ with $n_1 + \tn$ data points in total, where $n_1$ points are sampled from the distribution $P_1$, and we let $\tS$ denote $\tn$ points from $P_2$.  2) $S_2$, containing $n_2$ points sampled from $P_2$. 3) $S_3$, containing $n_3$ points sampled from $P_3$. 

We identify two types of potential errors: (1) misaligning points in $\tS$ with the text features of Class 1, and (2) incorrectly enforcing $\tS$ to be far from the text features of Class 2.

For simplicity, we focus on the first term in $\LL_{Align}$, as the second term follows a similar structure. Let $f_t^k$ denote the text feature of Class $k$, where $k \in {1, 2, 3}$. 
Denote
\begin{equation} 
\begin{split}
&\LL^1_{Align} := \sum\limits_{q=1}^{3} \lambda  \sum\limits_{m \neq q}^{}   \sum_{i\in S_q}\sum_{j\in S_m} \\
&~\left[\frac{1}{1+\exp((\|f_t^q - g^s_t(j)\|^2  - \|f_t^q - g^s_t(i)\|^2)/\lambda)} \right].    
\end{split} 
\end{equation}

Let
\begin{equation}
\LL_{q,m} :=  \lambda \sum_{i\in S_q}\sum_{j\in S_m} \ell^q(i,j),
\end{equation}
where 
\begin{equation}
\ell^q(i,j) = \frac{1}{1+\exp((\|f_t^q - g^s_t(j)\|^2  - \|f_t^q - g^s_t(i)\|^2)/\lambda)}.   
\end{equation} 

Rearranging the terms, we can rewrite the loss function as
\begin{equation}
\begin{split}
\LL^1_{Align} = \LL_{1,2} + \LL_{1,3} + \LL_{2,1} + \LL_{2,3} + \LL_{3,1} + \LL_{3,2},
\end{split}
\end{equation}
where 
\begin{equation}
\LL_{1,2} =  \lambda \sum_{i\in S_1/\tS}\sum_{j\in S_2} \ell^1(i,j) +  \lambda \underbrace{\sum_{i\in\tS}\sum_{j\in S_2} \ell^1(i,j)}\limits_{(A)}. 
\end{equation}

\begin{equation}
\LL_{1,3} =  \lambda \sum_{i\in S_1/\tS}\sum_{j\in S_3} \ell^1(i,j) +  \lambda \underbrace{\sum_{i\in\tS}\sum_{j\in S_3} \ell^1(i,j)}\limits_{(B)}. 
\end{equation}

\begin{equation}
\LL_{2,1} =  \lambda \sum_{i\in S_2}\sum_{j\in S_1/\tS} \ell^2(i,j) +  \lambda \underbrace{\sum_{i\in S_2}\sum_{j\in \tS} \ell^2(i,j) }\limits_{(C)}. 
\end{equation}

\begin{equation}
\LL_{2,3} =  \lambda \sum_{i\in S_2}\sum_{j\in S_3} \ell^2(i,j) 
\end{equation}

\begin{equation}
\LL_{3,1} =  \lambda \sum_{i\in S_3}\sum_{j\in S_1/\tS} \ell^3(i,j) + \lambda \underbrace{\sum_{i\in S_3}\sum_{j\in \tS} \ell^3(i,j) }\limits_{(D)} 
\end{equation}

\begin{equation}
\LL_{3,2} =  \lambda \sum_{i\in S_3}\sum_{j\in S_2} \ell^3(i,j) 
\end{equation}

We observe that the noisy subset $\tS$ is only involved in terms A, B, C, and D. Although term D involves $\tS$, it does not lead to misalignment, as it merely encourages the text of Class 3 to be similar to other text from Class 3 and dissimilar to $\tS$. Since $\tS$ is generated from $P_2$, this is a valid operation. Terms A and C can be addressed in the following theorem.
\begin{theorem} For the data sets generated as described above and the loss function defined accordingly, the terms A and C are equal to constants in expectation, i.e., 
\begin{equation} 
\E_{S_1, S_2, S_3} [A] = \E_{S_1, S_2, S_3} [C] = 1.  
\end{equation} 
\end{theorem}
\begin{proof}

For term A, we have 
\begin{equation}
\begin{split}
&\E_{\tS, S_2} \left[\lambda \sum_{i\in\tS}\sum_{j\in S_2} \ell^1(i,j) \right] = 
\lambda \tn n_2  \E_{i\in P_2}\E_{j\in P_2} \ell^1(i,j)  \\
& = \lambda \tn n_2 \E_{i\in P_2}\E_{j\in P_2} \frac{\ell^1(i,j) + \ell^1(j,i)}{2}, 
\end{split}
\end{equation}
where 
\begin{equation}
\begin{split}
&\frac{\ell^1(i,j) + \ell^1(j,i)}{2} \\
&= \frac{1}{1+\exp((\|f_t^1 - g^s_t(j)\|^2  - \|f_t^1 - g^s_t(i)\|^2)/\lambda)} \\
&~~~ + \frac{1}{1+\exp((\|f_t^1 - g^s_t(i)\|^2  - \|f_t^1 - g^s_t(j)\|^2)/\lambda)} \\
& =  \frac{\exp((\|f_t^1 - g^s_t(i)\|^2 - \|f_t^1 - g^s_t(j)\|^2 )/\lambda)}{1+\exp(( \|f_t^1 - g^s_t(i)\|^2 - \|f_t^1 - g^s_t(j)\|^2)/\lambda)} \\
&~~~ + \frac{1}{1+\exp((\|f_t^1 - g^s_t(i)\|^2  - \|f_t^1 - g^s_t(j)\|^2)/\lambda)} \\ 
&=1.
\end{split} 
\end{equation}
Similarly, for term C we obtain that
\begin{equation}
\begin{split}
&\E_{S_2, \tS} \left[\lambda \sum_{i\in S_2}\sum_{j\in \tS} \ell^2(i,j) \right] = 
\lambda n_2 \tn  \E_{i\in P_2}\E_{j\in P_2} \ell^2(i,j)  \\
& = \lambda n_2 \tn  \E_{i\in P_2}\E_{j\in P_2} \frac{\ell^2(i,j) + \ell^2(j,i)}{2}, 
\end{split}
\end{equation}
where 
\begin{equation}
\begin{split}
&\frac{\ell^2(i,j) + \ell^2(j,i)}{2} \\
&= \frac{1}{1+\exp((\|f_t^2 - g^s_t(j)\|^2  - \|f_t^2 - g^s_t(i)\|^2)/\lambda)} \\
&~~~ + \frac{1}{1+\exp((\|f_t^2 - g^s_t(i)\|^2  - \|f_t^2 - g^s_t(j)\|^2)/\lambda)} \\
& =  \frac{\exp((\|f_t^2 - g^s_t(i)\|^2 - \|f_t^2 - g^s_t(j)\|^2 )/\lambda)}{1+\exp(( \|f_t^2 - g^s_t(i)\|^2 - \|f_t^2 - g^s_t(j)\|^2)/\lambda)} \\
&~~~ + \frac{1}{1+\exp((\|f_t^2 - g^s_t(i)\|^2  - \|f_t^2 - g^s_t(j)\|^2)/\lambda)} \\ 
&=1.
\end{split} 
\end{equation} 
\end{proof}


For term B, which is given by
\begin{equation}
\begin{split}
& \sum_{i\in\tS}\sum_{j\in S_3} \ell^1(i,j) = \\
& \frac{1}{1+\exp((\|f_t^1 - g^s_t(j)\|^2  - \|f_t^1 - g^s_t(i)\|^2)/\lambda)},
\end{split}
\end{equation}
note that $\|f_t^1 - g^s_t(i)\|^2$ represents a misalignment term, but it can be partially balanced by $\|f_t^1 - g^s_t(j)\|^2$. Additionally, the term B does not exist in the case of a binary classification problem.
\subsection{Extra Ablation Study for Calibrated Alignment Loss}
\label{app:cali_vs_triplet}
In this subsection, we compare our results with those obtained using triplet losses as alignment losses. Although triplet losses also consider both positive and negative pairs, most of them do not satisfy the symmetric property, making them less robust to noisy features. The results are summarized in Table \ref{tab:triplet}.

Specifically, the triplet alignment losses are developed based on popular triplet loss formulations, as follows. 
First, following the work of \cite{schroff2015facenet}, we define:
\begin{equation}
\scriptsize
\begin{split}
\mathcal{L}_{\text{T,1}} = &\frac{1}{B} \sum\limits_{i\in B} \max(\|f_t(i) - g^s_t(i)\|^2 - \|f_t(i) - g^s_t(i^-)\|^2 + m, 0) \\
+ &\frac{1}{B} \sum\limits_{i\in B} \max(\|f_s(i) - g^t_s(i)\|^2 - \|f_s(i) - g^t_s(i^-)\|^2 + m, 0),
\end{split}
\end{equation} 
which $m$ is a margin term. It is not globally symmetric due to $\max(\cdot, 0)$ function.

Second, following \cite{dong2018triplet,hermans2017defense}, we define
\begin{equation}
\scriptsize
\begin{split}
\mathcal{L}_{\text{T,2}} = &\frac{1}{B} \sum\limits_{i\in B} \log\frac{1}{1 + \exp((\|f_t(i) - g^s_t(i^-)\|^2-\|f_t(i) - g^s_t(i)\|^2)/\lambda)} \\
+ &\frac{1}{B} \sum\limits_{i\in B} \log\frac{1}{1 + \exp((\|f_s(i) - g^t_s(i^-)\|^2-\|f_s(i) - g^t_s(i)\|^2)/\lambda)},
\end{split}
\end{equation} 
which is non-symmetric due to the $\log$ function.

Third, following \cite{hoffer2015deep}, we define 
\begin{equation}
\scriptsize
\begin{split}
\mathcal{L}_{\text{T,3}} = &\frac{\lambda}{B} \sum\limits_{i\in B} \left(\frac{\exp(\|f_t(i) - g^s_t(i)\|_2}{\exp(\|f_t(i) - g^s_t(i)\|_2 + \exp((\|f_t(i) - g^s_t(i^-)\|_2)}\right)^2 \\
+ &\frac{\lambda}{B} \sum\limits_{i\in B} \left(\frac{\exp(\|f_s(i) - g^t_s(i)\|_2)}{\exp(\|f_s(i) - g^t_s(i)\|_2) + \exp(\|f_s(i) - g^t_s(i^-)\|_2)}\right)^2,
\end{split}
\end{equation} 
which is non-symmetric due to the squared function. 

Fourth, following \cite{kumar2016learning}, we define
\begin{equation}
\scriptsize
\begin{split}
\mathcal{L}_{\text{T,4}} = &\frac{1}{B} \sum\limits_{i\in B} \max\left(1-\frac{ \|f_t(i) - g^s_t(i^-)\|^2}{\|f_t(i) - g^s_t(i)\|^2 + m}, 0\right) \\
+ &\frac{1}{B} \sum\limits_{i\in B} \max\left(1-\frac{\|f_s(i) - g^t_s(i^-)\|^2}{\|f_s(i) - g^t_s(i)\|^2 + m}, 0\right),
\end{split}
\end{equation} 
which is also non-symmetric. 

In the experiments of this subsection, the only distinction between our method and the others lies in the formulation of the alignment loss. As shown in Table \ref{tab:triplet}, although most of these methods outperform the baselines in the literature of ZSSAR, they perform significantly worse than ours with the calibrated alignment loss due to their absence of symmetry. This emphasizes the effectiveness of our alignment loss design.

\begin{table}[htbp]
\scriptsize
\renewcommand\arraystretch{1.1}
\centering
\caption{ZSL accuracy with different alignment loss.}
\vspace{5pt}
\setlength\tabcolsep{4.0pt}
{
\begin{tabular}{c|cc|cc}
\hline
\multirow{2}{*}{\shortstack{Alignment \\ Loss}} & \multicolumn{2}{c|}{NTU-60 (ACC,\%)} & \multicolumn{2}{c}{NTU-120 (ACC,\%)} \\ \cline{2-5} 
 & 55/5 split & 48/12 split & 110/10 split & 96/24 split \\ \hline
$\LL_{T,1}$ &84.4 &45.3 &\textbf{\textcolor{blue}{72.7}} &58.6 \\
$\LL_{T,2}$ &79.9 &32.0 & 59.1 & 38.7 \\
$\LL_{T,3}$ &83.8 &\textbf{\textcolor{blue}{49.5}} &71.8 & \textbf{\textcolor{blue}{60.7}} \\
$\LL_{T,4}$ &\textbf{\textcolor{blue}{85.3}} &42.2 &69.0 &49.7\\
\hline
\textbf{Ours} &\textbf{\textcolor{red}{86.9}} & \textbf{\textcolor{red}{57.2}} & \textbf{\textcolor{red}{74.4}} & \textbf{\textcolor{red}{62.5}} \\
\bottomrule 
\end{tabular}
}
\label{tab:triplet}
\vspace{-5pt}
\end{table}

\section{Frequency-based Representation Analysis for Skeleton Sequences}
\label{app:dct_proof}

\subsection{Motivation}
The Discrete Cosine Transform (DCT) enables lossless feature enhancement through energy-preserving manipulation. The key sight is that the strict energy preservation of DCT and Inverse-DCT (IDCT) between the frequency and time domains: \textbf{enhanced components in the frequency domain can be transferred to the time-domain features through IDCT without information loss}. This allows dual semantic enhancements: 1) amplifying low-frequency coefficients enhances global motion patterns (e.g., overarching torso coordination), 2) refining high-frequency components preserves fine-grained kinematics (e.g., hand trajectories) while mitigating the noise. Moreover, this energy-invariant enhancement provides richer information representations for further alignment, where cross-modal correspondences can be learned from both global and local action semantics. 

\subsection{Frequency Domain Representation and Energy Preservation Proof}
Let $\mathbf{s}\in \mathbb{R}^{J\times C\times F}$ denote a skeleton sequence in the time domain, where $J$ is the number of body joints (e.g., 25 joints in NTU-RGB+D dataset), $C$ is the number of coordinate dimensions ($C=3$ for $x,y,z$ coordinates), and $F$ is the temporal length (number of frames). The frequency-domain representation $\mathbf{C}\in \mathbb{R}^{J\times C\times F}$ is obtained through the orthogonal DCT. For each joint $j\in\{1,\ldots,J\}$, coordinate $c\in\{1,\ldots,C\}$, and frequency index $i\in\{0,\ldots,F-1\}$, the transformation is defined as:
\begin{equation}
C_{j,c,i} = \sum_{f=0}^{F-1} s_{j,c,f} \cdot \phi_i(f)
\label{eq:dct_forward}
\end{equation}
where the normalized DCT basis functions $\phi_i(f)$ are given by:
\begin{equation}
\phi_i(f) = \sqrt{\frac{2 - \delta_{i0}}{F}} \cdot \cos\!\left[\frac{\pi}{F}\left(f + \frac{1}{2}\right)i\right],
\end{equation}
with $\delta_{i0}$ denoting the Kronecker delta function (i.e., $\delta_{i0}=1$ when $i=0$ and $\delta_{i0}=0$ otherwise), and $f\in\{0,\ldots,F-1\}$. 

For any joint $j$ and coordinate $c$, the energy equivalence between the time and frequency domains is proved as follows:
\begin{equation}
\begin{aligned}
E_{\text{freq},j,c} &= \sum_{i=0}^{F-1} C_{j,c,i}^2 \\
&= \sum_{i=0}^{F-1} \left( \sum_{f=0}^{F-1} s_{j,c,f}\,\phi_i(f) \right)^2 \\
&= \sum_{i=0}^{F-1} \sum_{f=0}^{F-1} \sum_{f'=0}^{F-1} s_{j,c,f}\, s_{j,c,f'}\, \phi_i(f)\phi_i(f') \\
&= \sum_{f=0}^{F-1} \sum_{f'=0}^{F-1} s_{j,c,f}\, s_{j,c,f'}\, \sum_{i=0}^{F-1} \phi_i(f)\phi_i(f') \\
&= \sum_{f=0}^{F-1} s_{j,c,f}^2 = E_{\text{time},j,c}.
\end{aligned}
\end{equation}
The orthogonality relationship \cite{rao2014discrete}
\[
\sum_{i=0}^{F-1} \phi_i(f)\phi_i(f') =
\begin{cases} 
1, & \text{if } f = f' \\
0, & \text{if } f \neq f'
\end{cases}
\]

eliminates cross-terms between different frames ($f\neq f'$). Consequently, the energy preservation holds globally:
\begin{equation}
\sum_{j=1}^{J} \sum_{c=1}^{C} \sum_{f=0}^{F-1} s_{j,c,f}^2 = \sum_{j=1}^{J} \sum_{c=1}^{C} \sum_{i=0}^{F-1} C_{j,c,i}^2.
\end{equation}

\subsection{Semantic Integrity with Frequency Adjustment}
Given modified coefficients \( C'_{j,c,i} = C_{j,c,i} \cdot g(i) \) with scaling function \( g(i) \), the reconstructed signal becomes:

\begin{equation}
s'_{j,c,f} = \sum_{i=0}^{F-1} C'_{j,c,i} \phi_i(f) = \sum_{i=0}^{F-1} g(i) C_{j,c,i} \phi_i(f)
\label{eq:dct_inverse}
\end{equation}

The modified energy preserves the relationship:

\begin{equation}
\begin{aligned}
E'_{\text{time},j,c} &= \sum_{f=0}^{F-1} (s'_{j,c,f})^2 \\
&= \sum_{f=0}^{F-1} \left( \sum_{i=0}^{F-1} g(i) C_{j,c,i} \phi_i(f) \right)^2 \\
&= \sum_{i=0}^{F-1} \sum_{k=0}^{F-1} g(i)g(k) C_{j,c,i}C_{j,c,k} \underbrace{\sum_{f=0}^{F-1} \phi_i(f)\phi_k(f)}_{\delta_{ik}} \\
&= \sum_{i=0}^{F-1} g(i)^2 C_{j,c,i}^2 = E'_{\text{freq},j,c}
\end{aligned}
\end{equation}

This derivation demonstrates three key properties: First, the orthogonal basis eliminates cross-frequency interference during adjustment (\( \delta_{ik} \) removes terms where \( i \neq k \)), ensuring distortion-free modifications. Second, energy redistribution follows \( E'_{\text{time}} = \sum_i g(i)^2 C_{i}^2 \), allowing controlled enhancement (\( g(i)>1 \)) or suppression (\( g(i)<1 \)) of specific frequency. Third, semantic integrity is maintained through the physical meaning of frequency components - low frequencies (\( i \leq \varphi \)) encode global motion trajectories, while high frequencies (\( i > \varphi \)) capture local kinematic details ($\varphi$ is the low-frequency threshold), enabling targeted manipulation without corrupting overall motion semantics.

\begin{table*}[htbp]
\centering
\scriptsize
\renewcommand\arraystretch{1.3}
\begin{tabular}{l|c|c}
\toprule
\textbf{Property} & \textbf{DCT} & \textbf{Wavelet} \\
\midrule
Energy Compaction     & Strong global compaction & Localized \\
Coefficient Control   & Easy frequency separation & Requires multi-scale design \\
Integration           & Simple matrix operations & Needs wavelet basis selection \\
Usage                 & Semantic enrichment & Fine-grained separation \\
\bottomrule
\end{tabular}
\caption{Comparison between DCT and Wavelet in terms of structural properties and usage for representation learning.}
\label{tab:dct_vs_wavelet}
\vspace{-6pt}
\end{table*}
\subsection{Frequency-based Enhancement Mechanism}
Since semantic information in skeleton motion is inherently tied to frequency components, higher energy indicates richer information, while energy distribution across frequencies highlights different motion scales. Thus, enhancing skeleton-based frequency components in the frequency domain enriches semantic representation in the time domain (proved above, semantic integrity is preserved during DCT-IDCT), leading to improved generalization in ZSL. This mechanism consists of two adjustments:

\textbf{Low-Frequency Enhancement.} 
The amplification term \( w_i\left(1 - \frac{i}{b}\right) \) is designed to emphasize fundamental movement patterns in skeletal dynamics. By progressively reducing the enhancement effect as frequency increases, this mechanism ensures that low-frequency components, which encode the overall motion structure, are strengthened without distorting the natural motion flow. For whole-body actions such as ``walking" or ``clapping," it enhances limb coordination and preserves joint continuity. 

\textbf{High-Frequency Suppression.} 
The attenuation term \( -w_i\left(1 - \frac{i - b}{b}\right) \) is designed to progressively reduce the suppression effect as frequency increases. This ensures that while high-frequency components are attenuated to mitigate noise and skeletal jitter, fine-grained and rapid motion details are not excessively diminished. The parameter \( b \) controls the rate of suppression decay, allowing higher frequency components to retain essential micro-movements, such as finger and wrist gestures. 

\subsection{Illustration}
We also provide the illustration example of our frequency-enhanced mechanism in Fig. \ref{fig: freq illustration}. Assume the number of the DCT coefficients is 20, the low-frequency threshold \(\varphi\) is 15.
As shown in the figure, in the low-frequency range (\( i \leq \varphi \)), the enhancement applied to the low-frequency coefficients gradually decreases, allowing a smooth transition while preserving global motion integrity. Meanwhile, in the high-frequency range (\( i > \varphi \)), the suppression of high-frequency coefficients diminishes progressively, allowing essential fine-grained motion details to be retained while mitigating noise.

\subsection{Code}
The key part of the implementation of the frequency-enhanced module in our method is presented in Fig. \ref{code1}. The code snippet provided illustrates the core mechanism of our frequency-aware enhancement strategy within the skeleton decoder. The codes for frequency adjustment with purely learnable weight are also provided in Fig. \ref{code2}. Extra ablation study and discussion are provided in the main paper.

\section{Justification for Choosing DCT}
\label{app:freq_discussion}
We adopt the Discrete Cosine Transform (DCT) as our frequency encoding method due to its strong energy compaction property and its ability to flexibly separate low- and high-frequency components. These characteristics make it particularly effective for semantic representation learning in zero-shot settings, where training data is limited and fine-grained generalization is critical. Specifically, DCT helps preserve global motion information while enabling localized modulation. This frequency-aware modulation enriches latent representations without requiring strict temporal alignment, aligning well with the post-encoded features.

As shown in Table \ref{tab:dct_vs_wavelet}, while wavelet transforms are also viable for signal analysis, they are primarily designed for multi-scale, localized analysis and often require more complex basis selection and hierarchical decomposition. In contrast, DCT is lightweight, easily integrable through matrix operations, and offers more straightforward control over frequency bands for modulation. Our use of DCT is not intended as a traditional frequency separation mechanism, as in prior fully-supervised methods\cite{chang2024wavelet, wu2024frequency}, but as a semantic enhancement strategy to improve generalization under zero-shot learning. 

\section{NTU-60 Dataset Action Index}
\label{app:index}
We also provide the list of action indices from the NTU-60 dataset in Table \ref{tab:ntu60_index}.

\begin{table*}[t]
\centering
\scriptsize
\renewcommand\arraystretch{1.2}
\setlength{\tabcolsep}{5pt}
\caption{NTU-60 action classes and their corresponding indices.}
\begin{tabular}{c|l}
\hline
\textbf{Index} & \textbf{Action} \\ 
\hline
1  & Drink water \\
2  & Eat meal \\
3  & Brush teeth \\
4  & Brush hair \\
5  & Drop \\
6  & Pick up \\
7  & Throw \\
8  & Sit down \\
9  & Stand up \\
10 & Clapping \\
11 & Reading \\
12 & Writing \\
13 & Tear up paper \\
14 & Put on jacket \\
15 & Take off jacket \\
16 & Put on a shoe \\
17 & Take off a shoe \\
18 & Put on glasses \\
19 & Take off glasses \\
20 & Put on a hat/cap \\
21 & Take off a hat/cap \\
22 & Cheer up \\
23 & Hand waving \\
24 & Kicking something \\
25 & Reach into pocket \\
26 & Hopping \\
27 & Jump up \\
28 & Phone call \\
29 & Play with phone/tablet \\
30 & Type on a keyboard \\
31 & Point to something \\
32 & Taking a selfie \\
33 & Check time (from watch) \\
34 & Rub two hands together \\
35 & Nod head/bow \\
36 & Shake head \\
37 & Wipe face \\
38 & Salute \\
39 & Put palms together \\
40 & Cross hands in front \\
41 & Sneeze/cough \\
42 & Staggering \\
43 & Falling down \\
44 & Headache \\
45 & Chest pain \\
46 & Back pain \\
47 & Neck pain \\
48 & Nausea/vomiting \\
49 & Fan self \\
50 & Punch/slap \\
51 & Kicking \\
52 & Pushing \\
53 & Pat on back \\
54 & Point finger \\
55 & Hugging \\
56 & Giving object \\
57 & Touch pocket \\
58 & Shaking hands \\
59 & Walking towards \\
60 & Walking apart \\
\hline
\end{tabular}
\label{tab:ntu60_index}
\end{table*}


\begin{figure*}[]
  \centering
  \includegraphics[width=0.8\linewidth]{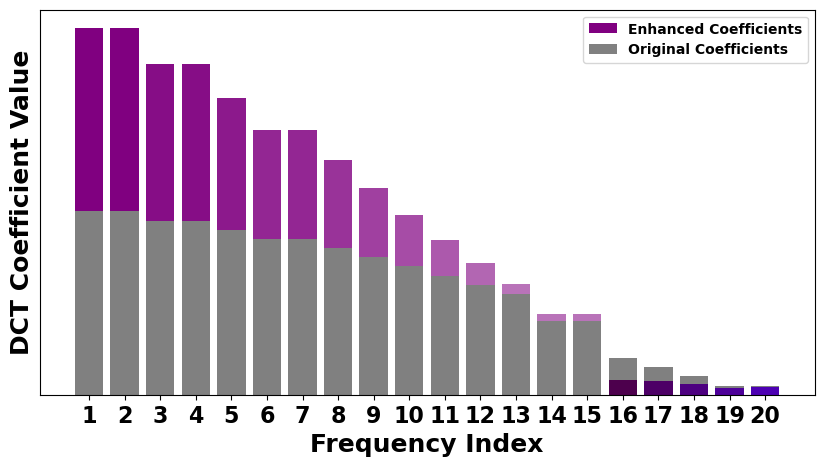}
  \caption{The illustration example of the frequency-enhanced method.}
  \label{fig: freq illustration}
  \vspace{-10pt}
\end{figure*}

\begin{table*}[]
\small
\centering
\caption{Implementation details and parameter settings.}
\label{tab:Implementation}
\begin{tabular}{lcc}
\hline
\textbf{Datasets} & \textbf{NTU-60} & \textbf{NTU-120} \\
\hline
\textbf{Skeleton Feature Extractor} & \multicolumn{2}{c}{Shift-GCN~\cite{cheng2020skeleton}} \\
\textbf{Text Feature Extractor} & \multicolumn{2}{c}{CLIP-ViT-B32/16~\cite{radford2021learning}} \\
\textbf{Latent Embedding Dim (Stage 1)} & 256 & 512 \\
\textbf{Latent Embedding Dim (Stage 2)} & 100 & 200 \\
\textbf{Optimizer} & \multicolumn{2}{c}{Adam} \\
\textbf{Learning Rate (Stage 2)} & \multicolumn{2}{c}{$1.0 \times 10^{-4}$} \\
\textbf{Batch Size (Stage 2)} & \multicolumn{2}{c}{64} \\
\textbf{Training Epochs (Stage 2)} & \multicolumn{2}{c}{1900} \\
\textbf{Unseen Class Features Dim (Stage 3)} & \multicolumn{2}{c}{500} \\
\textbf{Unseen Classifier Epochs (Stage 3)} & \multicolumn{2}{c}{300} \\
\textbf{Unseen Classifier Learning Rate} & \multicolumn{2}{c}{$1.0 \times 10^{-3}$} \\
\textbf{Classification Gate} & \multicolumn{2}{c}{Logistic Regression (LBFGS, $C=1$)} \\
\textbf{Frequency Module} & \multicolumn{2}{c}{DCT-IDCT~\cite{ahmed1974discrete}} \\
\textbf{Frequency Parameters} & \multicolumn{2}{c}{$\varphi=35$, $b=30$} \\
\textbf{Semantic Descriptions} & \multicolumn{2}{c}{GPT-4 Generated (LD+GD)} \\
\textbf{Calibrated Loss $\alpha$} & \multicolumn{2}{c}{0.1} \\
\textbf{Calibrated Loss $\lambda$} & \multicolumn{2}{c}{100} \\
\textbf{Hardware} & \multicolumn{2}{c}{NVIDIA A100 $\times$ 1} \\
\hline
\end{tabular}
\vspace{-5pt}
\end{table*}

\begin{table*}[]
\centering
\renewcommand{\arraystretch}{1.4}
\setlength{\tabcolsep}{8pt}
\caption{Examples of action descriptions between baseline and our method.}
\label{tab:semantic description}
\resizebox{\textwidth}{!}{
\begin{tabular}{|l|p{4.5cm}|p{6cm}|p{6cm}|}
\hline
\textbf{Action} & \textbf{Baseline Description} & \textbf{\textcolor{globalcolor}{Global} Description (Ours)} & \textbf{\textcolor{localcolor}{Local} Description (Ours)} \\ \hline
Eating Meal/Snack & to put food in your mouth, bite it, and swallow it & to \textcolor{globalcolor}{pick up food with your hand or utensil}, \textcolor{globalcolor}{move it to the mouth}, and \textcolor{globalcolor}{chew} & \textcolor{localcolor}{pinch and move the hand up to the head} \\ \hline
Brushing Teeth & to clean, polish, or make teeth smooth with a brush & to \textcolor{globalcolor}{move a toothbrush back and forth} inside your mouth & \textcolor{localcolor}{move the hand up to the head}, then \textcolor{localcolor}{tremble the wrist} \\ \hline
Brushing Hair & to clean, polish, or make hair smooth with a brush & to \textcolor{globalcolor}{run a brush or comb through your hair} to smooth it & \textcolor{localcolor}{move the hand up to the head}, then \textcolor{localcolor}{move the hand downward} \\ \hline
Dropping an Object & to allow something to fall by accident from your hands & to \textcolor{globalcolor}{release an object}, \textcolor{globalcolor}{letting it fall freely to the ground} & \textcolor{localcolor}{release the hand in front of the middle of the body} \\ \hline
\end{tabular}
}
\end{table*}


\clearpage
\clearpage
\definecolor{vscodeblue}{RGB}{86,156,214}
\definecolor{vscodegreen}{RGB}{78,201,176}
\definecolor{vscodepurple}{RGB}{197,134,192}
\definecolor{vscodegray}{RGB}{155,155,155}
\definecolor{bgcolor}{RGB}{240,248,255} 

\begin{tcolorbox}[
    boxrule=0pt,
    arc=0pt,
    outer arc=0pt,
    left=5mm,
    right=5mm,
    top=5mm,
    bottom=5mm,
    width=\textwidth,
    breakable
]

\begin{lstlisting}[
    language=Python,
    basicstyle=\ttfamily\small,
    keywordstyle=\color{vscodeblue},
    commentstyle=\color{vscodegreen}\itshape,
    stringstyle=\color{vscodepurple},
    numbers=left,
    numberstyle=\tiny\color{vscodegray},
    breaklines=true,
    showstringspaces=false,
    frame=none
]
# x = input data
# dct = Discrete Cosine Transform function
# b = adjusting parameter
# freq_weight = learnable weight for frequency
# split_freq = threshold for low- and high-frequency adjustment
    def dct_enhance(self, x):
        # Apply DCT to transform input to the frequency domain
        x_dct = dct.dct(x, norm='ortho')
        # Frequency enhancement
        for i in range(self.length_input):
            start = self.split_points[i]
            end = self.split_points[i + 1]
            freq_weight = self.freq_weight[i]
            # Low-frequency adjustment
            if end <= self.split_freq:  
                # Scaling function for low frequency 
                decay_factor = 1 - i / self.b
                x_dct[:, start:end] *= (1 + freq_weight * decay_factor)
            # High-frequency adjustment
            else:  
                # Scaling function for high frequency 
                decay_factor = 1 - (i - self.b) / self.b
                x_dct[:, start:end] *= (1 - freq_weight  * decay_factor)
        # Inverse DCT to transform back to the time domain
        return dct.idct(x_dct, norm='ortho')

\end{lstlisting}
\captionof{figure}{PyTorch codes for frequency enhancement in the encoder.}
\label{code1}
\end{tcolorbox}

\begin{tcolorbox}[
    boxrule=0pt,
    arc=0pt,
    outer arc=0pt,
    left=5mm,
    right=5mm,
    top=5mm,
    bottom=5mm,
    width=\textwidth,
    breakable
]

\begin{lstlisting}[
    language=Python,
    basicstyle=\ttfamily\small,
    keywordstyle=\color{vscodeblue},
    commentstyle=\color{vscodegreen}\itshape,
    stringstyle=\color{vscodepurple},
    numbers=left,
    numberstyle=\tiny\color{vscodegray},
    breaklines=true,
    showstringspaces=false,
    frame=none
]
    def dct_enhance(self, x):
        # Apply DCT to transform input to frequency domain
        x_dct = dct.dct(x, norm='ortho')
        for i in range(self.length_input):
            start = self.split_points[i]
            end = self.split_points[i + 1]
            freq_weight = self.freq_weight[i]
            # Apply learnable weight directly
            x_dct[:, start:end] *= freq_weight
        # Inverse DCT to transform back to time domain
        return dct.idct(x_dct, norm='ortho')

\end{lstlisting}
\captionof{figure}{PyTorch codes for frequency enhancement with pure learnable weights.}
\label{code2}
\end{tcolorbox}

\end{document}